\documentclass[letterpaper]{article} 

\usepackage{times}  
\usepackage{helvet}  
\usepackage{courier}  
\usepackage{url}  
\usepackage{graphicx}  
\usepackage[utf8]{inputenc} 
\usepackage[T1]{fontenc}    
\usepackage{booktabs}       
\usepackage{amsfonts}       
\usepackage{nicefrac}       
\usepackage{microtype}      
\usepackage{tikz}
\usetikzlibrary{shapes.misc, positioning}
\usepackage{thmtools,thm-restate}
\usepackage{parskip}
\usepackage{authblk}
\usepackage{amsmath}  
\usepackage{amsthm}
\usepackage[numbers,sort]{natbib}

\usepackage{floatrow}
\newfloatcommand{capbtabbox}{table}[][\FBwidth]
\usepackage{basic}

\def \codeURL {\url{https://github.com/SprocketLab/WS-Structured-Prediction}}

 \author[$\dagger$]{Harit Vishwakarma}
 \author[$\dagger$]{Nicholas Roberts}
 \author[$\dagger$]{Frederic Sala}

 \affil[$\dagger$]{Dept. of Computer Sciences, University of Wisconsin-Madison}
 \affil[ ]{\footnotesize{\texttt{\{hvishwakarma, nick11roberts, fredsala\}@cs.wisc.edu}}}

\usepackage[utf8]{inputenc} 
\usepackage[T1]{fontenc}    
\usepackage{hyperref}       
\usepackage{url}            
\usepackage{booktabs}       
\usepackage{amsfonts}       
\usepackage{nicefrac}       
\usepackage{amsmath}
\usepackage{amssymb}
\usepackage{microtype}      
\usepackage{xcolor}         

\usepackage{amsthm}
\usepackage{thmtools, thm-restate}
\usepackage{amsmath}
\usepackage{amssymb}
\usepackage{mathtools}
\usepackage{bm}
\usepackage{booktabs}       
\usepackage{enumitem}
\usepackage{wrapfig}
\usepackage{colortbl}
\usepackage{thm-restate}

\usepackage{algorithm,algpseudocode}
\algnewcommand{\LeftComment}[1]{\Statex \(\triangleright\) #1}
\renewcommand{\algorithmicrequire}{\textbf{Input:}}
\renewcommand{\algorithmicensure}{\textbf{Output:}}
\usepackage{tikz}

\newtheorem{lemma}{Lemma}

\newtheorem{definition}{Definition}

\newtheorem*{theorem*}{Theorem}
\def \E {\mathbb{E}}

\def \R {\mathbb{R}}
\def \lf {\lambda}
\def \M {\mathcal{M}}
\def \Y {\mathbb{Y}}

\def \P {\mathbb{P}}

\def \hn {\hat{\nu}}

\def \b0 {{\bf 0}}

\def \bb {{\bf b }}

\def \bd {{\bf d }}

\def \bu {{\bf u }}
\def \bv {{\bf v }}

\def \bx {{\bf x }}
\def \by {{\bf y }}
\def \bz {{\bf z }}

\def \bC {{\bf C }}
\def \bD {{\bf D }}

\def \bI {{\bf I }}

\def \bK {{\bf K }}
\def \bL {{\bf L }}
\def \bM {{\bf M }}

\def \bP {{\bf P }}
\def \bQ {{\bf Q }}

\def \bT {{\bf T }}

\def \bU {{\bf U }}
\def \bV {{\bf V }}

\def \bZ {{\bf Z }}

\def \cB {\mathcal{B}}

\def \cF {\mathcal{F}}

\def \cO {\mathcal{O}}

\def \cX {\mathcal{X}}
\def \cY {\mathcal{Y}}

\def \psR {\mathbb{R}^{d^+,d^-}}

\def \blm {\boldsymbol {\lambda }}
\def \bmu {\boldsymbol {\mu }}

\def \bth {\boldsymbol {\theta }}

\def \beps {\boldsymbol {\epsilon}}

\def \bSg {\boldsymbol {\Sigma}}

\newcommand{\ldef}{\vcentcolon=}
\newcommand{\rdef}{=\vcentcolon}

\newcommand{\eqname}[1]{\tag*{#1}}

\DeclareMathOperator*{\argmin}{arg\,min}

\title{\textbf{Lifting Weak Supervision To Structured Prediction}}

\begin{document}

\maketitle

\begin{abstract}
Weak supervision (WS) is a rich set of techniques that produce pseudolabels by aggregating easily obtained but potentially noisy label estimates from a variety of sources.
WS is theoretically well understood for binary classification, where simple approaches enable consistent estimation of pseudolabel noise rates.
Using this result, it has been shown that downstream models trained on the pseudolabels have generalization guarantees nearly identical to those trained on clean labels.
While this is exciting, users often wish to use WS for \emph{structured prediction}, where the output space consists of more than a binary or multi-class label set: e.g. rankings, graphs, manifolds, and more. 
Do the favorable theoretical properties of WS for binary classification lift to this setting?
We answer this question in the affirmative for a wide range of scenarios.
For labels taking values in a finite metric space, we introduce techniques new to weak supervision based on pseudo-Euclidean embeddings and tensor decompositions, providing a nearly-consistent noise rate estimator.
For labels in constant-curvature Riemannian manifolds, we introduce new invariants that also yield consistent noise rate estimation.
In both cases, when using the resulting pseudolabels in concert with a flexible downstream model, we obtain generalization guarantees nearly identical to those for models trained on clean data.
Several of our results, which can be viewed as robustness guarantees in structured prediction with noisy labels, may be of independent interest. 
Empirical evaluation validates our claims and shows the merits  of the proposed method\footnote{\codeURL}. 
\end{abstract}

\section{Introduction}

Weak supervision (WS) is an array of methods used to construct pseudolabels for training supervised models in label-constrained settings. The standard workflow \cite{Ratner16, Ratner18, fu2020fast} is to assemble a set of cheaply-acquired labeling functions---simple heuristics, small programs, pretrained models, knowledge base lookups---that produce multiple noisy estimates of what the true label is for each unlabeled point in a training set. These noisy outputs are modeled and aggregated into a single higher-quality pseudolabel. Any conventional supervised end model can be trained on these pseudolabels. This pattern has been used to deliver excellent performance in a range of domains in both research and industry settings \cite{Dunnmon20, re2019overton, Bach20}, bypassing the need to invest in large-scale manual labeling. Importantly, these successes are usually found in binary or small-cardinality classification settings. 

While exciting, users often wish to use weak supervision in \emph{structured prediction} (SP) settings, where the output space consists of more than a binary or multiclass label set \cite{Bakir, Kuleshov15}. In such cases, there exists meaningful algebraic or geometric structure to exploit. Structured prediction includes, for example, learning rankings used for recommendation systems \cite{Korba18}, regression in metric spaces \cite{Petersen19}, learning on manifolds \cite{Rudi18}, graph-based learning \cite{Graber19}, and more. 

An important advantage of WS in the standard setting of binary classification is that it sometimes yields models with nearly the same generalization guarantees as their fully-supervised counterparts. Indeed, the penalty for using pseudolabels instead of clean labels is only a multiplicative constant. This is a highly favorable tradeoff since acquiring more unlabeled data is easy. This property leads us to ask the key question for this work: \textbf{does weak supervision for structured prediction preserve generalization guarantees?} We answer this question in the affirmative, justifying the application of WS to settings far from its current use.

Generalization results in WS rely on two steps \cite{Ratner19, fu2020fast}: (i) showing that the estimator used to learn the model of the labeling functions is consistent, thus recovering the noise rates for these noisy voters, and (ii) using a noise-aware loss to de-bias end-model training \cite{natarajan2013Noise}. Lifting these two results to structured prediction is challenging. The only available weak supervision technique suitable for SP is that of \cite{Shin22}. It suffers from several limitations. First, it relies on the availability of isometric embeddings of metric spaces into $\mathbb{R}^d$---but does not explain how to find these. Second, it does not tackle downstream generalization at all. We resolve these two challenges.

We introduce results for a wide variety of structured prediction problems, requiring only that the labels live in some metric space. We consider both finite and continuous (manifold-valued) settings. For finite spaces, we apply two tools that are new to weak supervision. The approach we propose combines isometric \emph{pseudo-Euclidean embeddings} with \emph{tensor decompositions}---resulting in a nearly-consistent noise rate estimator. In the continuous case, we introduce a label model suitable for the so-called \emph{model spaces}---Riemannian manifolds of constant curvature---along with extensions to even more general spaces. In both cases, we show generalization results when using the resulting pseudolabels in concert with a flexible end model from \cite{Ciliberto2016, Rudi18}.


\paragraph{Contributions:}
\begin{itemize}
  \setlength\itemsep{0em}
    \item New techniques for performing weak supervision in finite metric spaces based on isometric pseudo-Euclidean embeddings and tensor decomposition algorithms,
    \item Generalizations to manifold-valued regression in constant-curvature manifolds, 
    \item Finite-sample error bounds for noise rate estimation in each scenario,
    \item Generalization error guarantees for training downstream models on pseudolabels,
    \item  Experiments confirming the theoretical results and showing improvements over \cite{Shin22}.
\end{itemize}

\section{Background and Problem Setup}
\label{sec:setup}
Our goal is to theoretically characterize how well learning with pseudolabels (built with weak supervision techniques) performs in structured prediction. We seek to understand the interplay between the noise in WS sources and the generalization performance of the downstream structured prediction model. We provide brief background and introduce our problem and some useful notation.

\subsection{Structured Prediction}
Structured prediction (SP) involves predicting labels in spaces with rich structure. Denote the label space by $\cY$. Conventionally $\cY$ is a set, e.g., $\cY = \{-1, +1\}$ for binary classification. In the SP setting, $\cY$ has some additional algebraic or geometric structure. In this work we assume that $\cY$ is a metric space with metric (distance) $d_\cY$. This covers many types of problems, including
\begin{itemize}
  \setlength\itemsep{0em}
    \item Rankings, where $\cY = S_\rho$, the symmetric group on $\{1, \ldots, \rho\}$, i.e., labels are permutations,
    \item Graphs, where $\cY = \mathcal{G}_\rho$, the space of graphs with vertex set $V = \{1, \ldots, \rho\}$, 
    \item Riemannian manifolds, including $\cY = \mathbb{S}_d$, the sphere, or $\mathbb{H}_d$, the hyperboloid. 
\end{itemize}

\paragraph{Learning and Generalization in Structured Prediction} 
In conventional supervised learning we have a dataset $\{(x_1,y_1), \ldots ,(x_n,y_n)\}$ of i.i.d. samples drawn from distribution $\rho$ over $\cX \times \cY$. 
As usual, we seek to learn a model that generalizes well to points not seen during training. Let $\cF = \{f:\cX \mapsto \cY\}$ be a family of functions from $\cX$ to $\cY$. Define the risk $R(f)$ for $f\in \cF$ and $f^*$ as
\begin{align}
    R(f) = \int_{\mathcal{X} \times \mathcal{Y}} d^2_{\mathcal{Y}}(f(x), y) d\rho(x, y) \qquad f^* \in \argmin_{f\in \cF} R(f).
    \label{eq:risk}
\end{align}
For a large class of settings (including all of those we consider in this paper), \cite{Ciliberto2016, Rudi18} have shown that the estimator $\hat{f}$ is consistent: 
\begin{align}
 \hat{f}(x) = \arg\min_{y \in \cY} F(x,y) \qquad F(x,y) \ldef \frac{1}{n}\sum_{i=1}^n \alpha_i(x)d^2_\cY(y,y_i),
\label{eq:f-hat}
\end{align}
where $\alpha(x) = (\bK+ \nu \bI)^{-1} \bK_x$. Here, $\bK$ is the kernel matrix for a p.d. kernel $k: \mathcal{X} \times \mathcal{X} \rightarrow \mathbb{R}$, so that $\bK_{i,j} = k(x_i, x_j)$, $(\bK_x)_i = k(x, x_i)$, and $\nu$ is a regularization parameter. The procedure here is to first compute the weights $\alpha$ and then to perform the optimization in \eqref{eq:f-hat} to make a prediction.

An exciting contribution of \cite{Ciliberto2016, Rudi18} is the generalization bound 
\[R(\hat{f}) \le R(f^*) + \cO(n^{-\frac{1}{4}}),\]
that holds with high probability, as long as there is no label noise. The key question we tackle is \emph{does the use of pseudolabels instead of true labels $y_i$ affect the generalization rate?} 





\subsection{Weak Supervision}
In WS, we cannot access \emph{any} of the ground-truth labels $y_i$. Instead we observe for each $x_i$ the noisy votes $\lambda_{1,i}, \ldots, \lambda_{m,i}$. These are $m$ weak supervision outputs provided by \emph{labeling functions} (LFs) $s_a$, where $s_a : \mathcal{X} \rightarrow \cY$ and $\lambda_{a,i} = s_a(x_i)$. A two step process is used to construct pseudolabels. First, we learn a \emph{noise model} (also called a label model) that determines how reliable each source $s_a$ is. That is, we must learn $\bth$ for $P_{\bth}(\lf_{1}, \lf_{2}, \ldots, \lf_m |y)$---without having access to any samples of $y$. Second, the noise model is used to infer a distribution (or its mode) for each point: $P_{\bth}(y_i | \lf_{1,i}, \dots, \lf_{m,i})$.

We adopt the noise model from \cite{Shin22}, which is suitable for our SP setting:
\begin{equation}
   P_{\bth}(\lambda_1, \ldots, \lambda_m|Y=y) = \frac{1}{Z} \exp\left(-\sum_{a=1}^m \theta_a d^2_\cY(\lambda_a,y) { - \sum_{(a,b) \in E} \theta_{a,b} d^2_{\mathcal{Y}}(\lambda_a, \lambda_b)}\right). 
   \label{eq:lbl-model-orig-a}
\end{equation}
$Z$ is the normalizing partition function, $\bth = [\theta_1, \ldots, \theta_m]^T > 0$ are \emph{canonical} parameters{, and $E$ is a set of correlations}. The model can be described in terms of the \emph{mean} parameters $\mathbb{E}[d^2_{\cY}(\lf_a, y)]$. Intuitively, if $\theta_a$ is large, the typical distance from $\lf_a$ to $y$ is small 
and the LF is reliable; if $\theta_a$ is small, the LF is unreliable. 
This model is appropriate for several reasons. It is an exponential family model with useful theoretical properties. It subsumes popular special cases of noise, including, for regression, zero-mean multivariate Gaussian noise; for permutations, a generalization of the popular Mallows model; for the binary case, it produces a close relative of the Ising model.

Our goal is to form estimates $\hat{\bth}$ in order to construct pseudolabels. One way to build such pseudolabels is to compute
$    \tilde{y} = \argmin_{z \in \cY} 1/m \sum_{a=1}^m \hat{\theta}_a d^2_{\cY}(z, \lambda_{a})$. 
Observe how the estimated parameters $\hat{\theta}_a$ are used to weight the labeling functions, ensuring that more reliable votes receive a larger weight. 

We are now in a position to state the main research question for this work: 

{\bf {Do there exist estimation approaches yielding $\hat{\bth}$ that produce pseudolabels $\tilde{y}$ that maintain the same generalization error rate $\cO(n^{-1/4})$ when used in \eqref{eq:f-hat}, or a modified version of \eqref{eq:f-hat}?}}


\section{Noise Rate Recovery in Finite Metric Spaces}
In the next two sections we handle finite metric spaces. Afterwards we tackle continuous (manifold-valued) spaces. 
We first discuss learning the noise parameters $\bth$, then the use of pseudolabels.

\paragraph{Roadmap} For finite metric spaces with $|\cY| = r $, we apply two tools new to weak supervision. First, we embed $\cY$ into a \emph{pseudo-Euclidean} space \cite{goldfarb85psuedo}. These spaces generalize Euclidean space, enabling isometric (distance-preserving) embeddings for any metric. Using pseudo-Euclidean spaces make our analysis slightly more complex, but we gain the isometry property, which is critical.

Second, we form three-way tensors from embeddings of observed labeling functions. Applying tensor product decomposition algorithms \cite{anandkumar2014tensor}, we can recover estimates of the mean parameters $\hat{\mathbb{E}}[d^2_{\cY}(\lf_a, y)]$ and ultimately $\hat{\theta}_a$. Finally, we reweight the model \eqref{eq:f-hat} to preserve generalization. 

The intuition behind this approach is the following. First, we need a technique that can provide consistent or nearly-consistent estimates of the parameters in the noise model. Second, we need to handle any finite metric space. Techniques like the one introduced in \cite{fu2020fast} handle the first---but do not work for generic finite metric spaces, only binary labels and certain sequences. Techniques like the one in \cite{Shin22} handle any metric space---but only have consistency guarantees in highly restrictive settings (e.g., it requires an isometric embedding, that the distribution over the resulting embeddings is isomorphic to certain distributions, the true label only takes on two values). Pseudo-Euclidean embeddings used with tensor decomposition algorithms meet both requirements

\begin{figure}
    \centering
    \includegraphics[width=\linewidth]{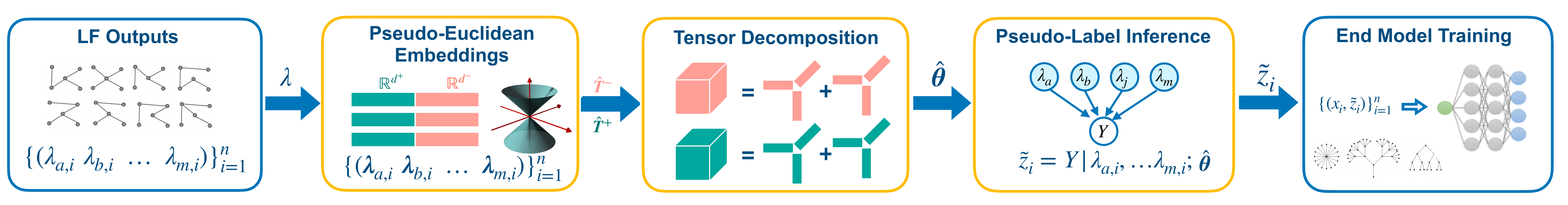}
    \vspace{-10pt}
    \caption{Illustration of our weak supervision pipeline for the finite label space setting.}
    \label{fig:my_label}
\end{figure}
\subsection{Pseudo-Euclidean Embeddings}

Our first task is to embed the metric space into a continuous space---enabling easier computation and potential dimensionality reduction.
A standard approach is multi-dimensional scaling (MDS) \cite{KruskalWish1978}, which embeds $\cY$ into $\mathbb{R}^d$. A downside of MDS is that not all metric spaces embed (isometrically) into Euclidean space, as the square distance matrix $\bD$ must be positive semi-definite. 

A simple and elegant way to overcome this difficulty is to instead use \emph{pseudo-Euclidean} spaces for embeddings. These pseudo-spaces do not require a p.s.d. inner product. As an outcome, any finite metric space can be embedded into a pseudo-Euclidean space with \emph{no distortion} \cite{goldfarb85psuedo}---so that distances are exactly preserved. Such spaces have been applied to similarity-based learning methods \cite{Pekalska01, laub2006information, Pekalska06}.
%
A vector $\bu$ in a pseudo-Euclidean space $\psR$ has two parts: $\bu^+ \in \R^{d^+}$ and $\bu^- \in \R^{d^-}$. The dot product and the squared distance between any two vectors $\bu,\bv$ are $\langle \bu, \bv\rangle_{\phi} = \langle \bu^+,\bv^+ \rangle - \langle \bu^-,\bv^- \rangle$ and $d^2_{\phi}(\bu,\bv) = ||\bu^{+}-\bv^{+}||_2^2 - ||\bu^{-}-\bv^{-}||_2^2$.  These properties enable isometric embeddings: the distance can be decomposed into two components that are individually induced from p.s.d. inner products---and can thus be embedded via MDS. Indeed, pseudo-Euclidean embeddings effectively run MDS for each component (see Algorithm~\ref{alg:finite} steps 4-9). To recover the original distance, we obtain $||\bu^{+}-\bv^{+}||_2^2$ and $ ||\bu^{-}-\bv^{-}||_2^2$ and subtract.

{\it Example:} To see why such embeddings are advantageous, we compare with a one-hot vector representation (whose dimension is $|\cY|$). Consider a tree with a root node and three branches, each of which is a path with $t$ nodes. Let $\cY$ be the nodes in the tree with the shortest-hops distance as the metric. The pseudo-Euclidean embedding dimension is just $d=3$;  see Appendix for more details. The one-hot embedding dimension is $d=|\cY| = 3t+1$---arbitrarily larger! 

Now we are ready to apply these embeddings to our problem. Abusing notation, we write $\blm_a$ and $\by$ for the pseudo-Euclidean embeddings of $\lf_a, y$, respectively. We have that $d^2_{\cY}(\lf_a, y) = d^2_{\phi}(\blm_a, \by)$, so that there is no loss of information from working with these spaces.
%
In addition, we write the mean as $\bmu_{a,y}=\E[\blm_a|\by]$ and the covariance as $\bSg_{a,y}$. 
Our goal is to obtain an accurate estimate $\hat{\bmu}_{a,y} =\hat{\E}[\blm_a|\by]$, which we will use to estimate the mean parameters $\E[d^2_{\cY}(\lambda_a,y)]$.
If we could observe $y$, it would be easy to empirically estimate  $\bmu_{a,y}$---but we do not have access to it. Our approach will be to apply tensor decomposition for multi-view mixtures \cite{anandkumar2014sample}.

\subsection{Multi-View Mixtures and Tensor Decompositions}
In a multi-view mixture model, multiple views $\{\lambda_{a}\}_{a=1}^m$ of a latent variable $Y$ are observed. These views are independent when conditioned on $Y$. 
We treat the positive and negative components $\blm_{a}^+ \in \R^{d^+}$ and $\blm_{a}^{-} \in \R^{d^-}$ of our pseudo-Euclidean embedding as separate multi-view mixtures:
\begin{equation}
    \blm_a^+|\by \sim \bmu_{a,y}^+ + \sigma \sqrt{d^+} \cdot \beps_a^+ \quad \text{ and } \quad \blm_a^-|\by \sim \bmu_{a,y}^- + \sigma \sqrt{d^-} \cdot \beps_a^- \qquad \forall a \in [m],
    \label{eq:proxy-lbl-model}
\end{equation}
where $\bmu^+_{a,y} = \E[\blm_a^+ | \by]$, $\bmu^-_{a,y} = \E[\blm_a^- | \by]$  and  $\beps_a^+, \beps_a^- $ are mean zero random vectors with covariances $\frac{1}{d^+} \bI_{d^+}, \frac{1}{d^-} \bI_{d^-}$ respectively. Here $\sigma^2$ is a proxy variance whose use is described in Assumption~\ref{asm:proxyvar}.

\begin{algorithm}[t]
\begin{algorithmic}[1]
\caption{Algorithm for Pseudolabel Construction} \label{alg:finite}
\Require{ Labeling function outputs $\bL = \{(\lambda_{1,i}, \ldots,  \lambda_{m,i})\}_{i=1}^n$,  Label Space $\cY=\{y_0,\ldots, y_{r-1}\}$}
\Ensure{ Pseudolabels for each data point $\bZ = \{\tilde{z}_i\}_{i=1}^n$}
 \Statex
 \LeftComment{ Step 1: Compute pseudo-Euclidean Embeddings}
 \State {Construct matrices $\bD \in \R^{r\times r}$, $\bD_{ij}= d^2_{\cY}(y_i,y_j)$ and $\bM\in \R^{r\times r}$, $ \bM_{ij} = \frac{1}{2}(\bD_{0i}^2 + \bD_{0j}^2 - \bD_{ij}^2) $ }
  \State{Compute eigendecomposition of $\bM$ and let $\bM = \bU \bC \bU^T$ }
  \State{Set $l^+, l^-$ be indices of positive and negative eigenvalues sorted by their magnitude}
 \State{Let $d^+ = |l^+|,\quad d^- = |l^-|$} i.e. the sizes of lists $l^+$ and $l^-$ respectively.
  \State{Construct permutation matrix $\bI_{perm} \in \R^{r\times (d^++d^-)}$ by concatenating $l^+, l^-$ in order }
 \State{$\bar{\bC} = \bC\bI_{perm}, \bar{\bU} =\bU \bI_{perm} $ }
 
  \State{$\Y = \bar{\bU}^T \bar{\bC}^{\frac{1}{2}} \in \R^{r\times (d^++d^-)}$ and let this define the mapping $g:\cY \mapsto \Y$ } 
 \Statex 
 \LeftComment{Step 2: Parameter Estimation Using Tensor Decomposition}
 \For{$a \gets 1$ to $m-3$} 
  \State{Obtain embeddings $\blm_{a,i} = g(\lambda_{a,i}), \blm_{b,i} = g(\lambda_{b,i}), \blm_{c,i} = g(\lambda_{c,i})  \quad \forall i \in [n] $ where $a,b,c$ are uncorrelated }
   \State{Construct tensors $\hat{\bT}^+$ and $\hat{\bT}^{-}$ as defined in  \eqref{eq:T-hat-split} for triplet $(a,b,c)$}
 \State{$\hat{\bmu}_{a,y}^+, \hat{\bmu}_{b,y}^{+},\hat{\bmu}_{c,y}^{+}$ = \texttt{TensorDecomposition}($\hat{\bT}^+$) }
 
 \State {$\hat{\bmu}_{a,y}^-, \hat{\bmu}_{b,y}^{-},\hat{\bmu}_{c,y}^{-}$ = \texttt{TensorDecomposition}($\hat{\bT}^-$)}
 \State {$s^+_{a,y} = \min_{z \in\{-1,+1\}} \phi(z\cdot\hat{\bmu}^+_{a,y},\by^+)$ and similarly  $s^+_{b,y},s^+_{c,y},s^-_{a,y},s^-_{b,y},s^-_{c,y}$ }
 \State {$\hat{\bmu}^+_{a,y} = s^+_{a,y}\cdot\hat{\bmu}^+_{a,y}$ and similarly correct signs of $\hat{\bmu}^+_{b,y},\hat{\bmu}^+_{c,y},\hat{\bmu}^-_{a,y},\hat{\bmu}^-_{b,y},\hat{\bmu}^-_{c,y}$ }
  \EndFor
\Statex
 
 \LeftComment{Step 3: Infer Pseudo-Labels}
   \State{$ \tilde{Z}^{(i)}=\tilde{z}_i \sim  Y | \lambda_a = \lambda_a^{(i)} ,\ldots \lambda_m = \lambda_m^{(i)} ; \hat{\bth}  $}
   \Statex 
   \State {\Return $\{\tilde{z}_i\}_{i=1}^n$}
 \end{algorithmic}

\end{algorithm}

We cannot directly estimate these parameters from  observations of $\blm_a$, due to the fact that $\by$ is not observed. However, we can observe various moments of the outputs of the LFs such as tensors of outer products of LF triplets. We require that for each $a$ such a triplet  exists. Then, 
\begin{equation}
\bT^+ \ldef  \E[ \blm_a^+ \otimes \blm_b^+ \otimes \blm_c^+] = \sum_{y\in \cY_s} w_{y} \bmu_{a,y}^+\otimes \bmu_{b,y}^+\otimes \bmu_{c,y}^+  \text{ and }  \hat{\bT}^+ \ldef  \frac{1}{n}\sum_{i=1}^n  \blm_{a,i}^+\otimes \blm_{b,i}^+\otimes \blm_{c,i}^+.
\label{eq:T-hat-split}
\end{equation} 
Here $w_y$ are the mixture probabilities (prior probabilities of $Y$) and $\cY_s = \{y:w_y>0\}$. We similarly define $\bT^-$ and $ \hat{\bT}^{-} $.
We then obtain estimates $\hat{\bmu}_{a,y}^+,\hat{\bmu}_{a,y}^-$ using an algorithm from \cite{anandkumar2014tensor} with minor modifications to handle pseudo-Euclidean rather than Euclidean space. The overall approach is shown in Algorithm~\ref{alg:finite}. We have three key assumptions for our analysis,
\begin{restatable}{assumption}{} The support of $P_Y$, i.e., $k = |\{y:w_y>0\}|$ 
and the label space $\cY$ is such that $\min(d^+,d^-)\ge k$,
$||\bmu^+_{a,y}||_2 = 1 ,||\bmu^-_{a,y}||_2 = 1$ for $a \in [m],y \in \cY$.
\label{asm:support}
\end{restatable}
\begin{restatable}{assumption}{}{(Bounded angle between $\bmu$ and $\by$}) Let $\omega(\bu,\bv)$ denote the angle between any two vectors $\bu,\bv$ in a Euclidean space. We assume that $\omega(\bmu^+_{a,y},\by^+) \in [0,\pi/2-c)$, $\omega(\bmu^-_{a,y},\by^-) \in [0,\pi/2-c)$ $\forall a \in [m]$, and $y \in \cY_s$, for some sufficiently small $c \in (0,\pi/4] $ such that $\sin(c) \ge \max(\epsilon_0(d^+),\epsilon_0(d^-))$, where $\epsilon_0(d)$ is defined for some $n > n_0$ samples in \eqref{eq:epsilon-plus}. 
\label{ass:angle-bounded} 
\end{restatable}

\begin{restatable}{assumption}{} $\sigma$ is such that the recovery error with model \eqref{eq:proxy-lbl-model} is at least as large as with \eqref{eq:lbl-model-orig-a} .
\label{asm:proxyvar}
\end{restatable}

These enable providing guarantees on recovering the mean vector magnitudes \eqref{asm:support} and signs \eqref{ass:angle-bounded} and simplify the analysis \eqref{asm:support}, \eqref{asm:proxyvar}; all three can be relaxed at the expense of a more complex analysis.

Our first theoretical result shows that we have near-consistency in estimating the mean parameters in \eqref{eq:lbl-model-orig-a}. We use standard notation $\tilde{\cO}$ ignoring logarithmic factors.

\begin{restatable}[]{theorem}{meanParamRecovery}
Let $\hat{\bmu}^+_{a,y}, \hat{\bmu}^-_{a,y}$ be the estimates of ${\bmu}^+_{a,y}, {\bmu}^-_{a,y}$ returned by Algorithm \ref{alg:finite} with input $\hat{\bT}^+,\hat{\bT}^-$ constructed using isometric pseudo-Euclidean embeddings (in $\R^{d^+,d^-}$). Suppose Assumptions 1 and 2 are met, a sufficiently large number of samples $n$ are drawn from the model in ~\eqref{eq:lbl-model-orig-a}, 
and $k=|\cY_s|$. Then there exists a constant $C_0>0$ such that with high probability $ \forall a \in [m]$ and $y\in \cY_s$,
\begin{align*}
   |\theta_a -\hat{\theta}_a| \le  C_0\Big|\E[d^2_\cY(\lambda_a,y)] -  \hat{\E}[d^2_\cY(\lambda_a,y)] \Big | \le   \epsilon(d^+) + \epsilon(d^-) , 
\end{align*}


where
\begin{equation}
       \epsilon(d) \ldef \begin{cases} \tilde{\cO}\Big ( k\sqrt{\frac{d}{n}} \Big  ) + \tilde{\cO}\Big ( \frac{\sqrt{k}}{d} \Big ) \quad & \text{ if } \, \sigma^2 = \Theta(1), \\ 
       \tilde{\cO}\Big ( \sqrt{\frac{k}{n}} \Big  ) + \tilde{\cO}\Big ( \frac{\sqrt{k}}{d} \Big ) \quad &\text{ if } \, \sigma^2 = \Theta(\frac{1}{d}).
       \end{cases}
       \label{eq:epsilon-plus}
   \end{equation}
\label{thm:meanfinite}
\end{restatable}
We  interpret Theorem~\ref{thm:meanfinite}. It is a nearly direct application of \cite{anandkumar2014sample}. There are two noise cases for $\sigma$. In the high-noise case, $\sigma$ is independent of dimension $d$ (and thus $|\cY|$). Intuitively, this means the average distance balls around each LF begin to overlap as the number of points grows---explaining the multiplicative $k$ term. If the noise scales down as we add more embedded points, this problem is removed, as in the low-noise case.
In both cases, the second error term comes from using the algorithm of \cite{anandkumar2014tensor} and is independent of the sampling error. Since $k=\Theta(d)$, this term goes down with $d$.
The first error term is due to sampling noise and goes to zero in the number of samples $n$.
Note the tradeoffs of using the embeddings. If we used one-hot encoding, $d = |\cY|$, and in the high-noise case, we would pay a very heavy cost for $\sqrt{d/n}$. However, while sampling error is minimized when using a very small $d$, we pay a cost in the second error term. This leads to a tradeoff in selecting the appropriate embedding dimension.

\section{Generalization Error for Structured Prediction in Finite Metric Spaces}
We have access to labeling function outputs $\lambda_{1,i},\ldots, \lambda_{m,i}$ for points $x_i$ and noise rate estimates $\hat{\theta}_a,\ldots,\hat{\theta}_m$. How can we use these to infer  unobserved labels $y$ in \eqref{eq:f-hat}?
Our approach is based on \cite{natarajan2013Noise, rooyen18Noise},where the underlying loss function is modified to deal with noise. Analogously, we modify \eqref{eq:f-hat} in such a way that the generalization guarantee is nearly preserved. 




\subsection{Prediction with Pseudolabels}
First, we construct the posterior distribution $P_{\hat{\bth}}(Y=y|\lambda)$. We use our estimated noise model $P_{\hat{\bth}}(\lambda|Y)$ and the prior $P(Y=y)$. We create pseudo-labels for each data point by drawing a random sample from the posterior distribution conditioned on the output of labeling functions:
$    \tilde{Z}^{(i)}=\tilde{z}_i \sim  Y | \lambda_a = \lambda_a^{(i)} ,\ldots, \lambda_m = \lambda_m^{(i)} ; \hat{\bth}   .$
We thus observe $(x_1,\tilde{z}_1),\ldots, (x_n,\tilde{z}_n) $ where $\tilde{z}_i$  is sampled as above. 
To overcome the effect of noise we create a perturbed version of the distance function using the noise rates, generalizing \cite{natarajan2013Noise}. This requires us to characterize the noise distribution induced by our inference procedure. In particular we seek the probability that $\tilde{Z}=y_j$ when the true label is $y_j$. This can be expressed as follows. 
Let $\cY^m$ denote the $m$-fold Cartesian product of $\cY$ and let $\Lambda_u = (\lambda_1^{(u)},\ldots, \lambda_m^{(u)})$ denote its $u^{th}$ entry. We write 
   \begin{equation}
   \bP_{ij} = P_{\bth}(\tilde{Z} = y_j | Y= y_i) = \sum_{ u=1}^{|\cY^m|} P_{\bth}(\tilde{Z}=y_j | \Lambda = \Lambda^{(u)} ) \cdot P_{\bth}( \Lambda = \Lambda^{(u)} |Y = y_i) .
   \label{eq:dist-P}
   \end{equation}
We define $\bQ_{ij} = P_{\hat{\bth}}(\tilde{Z} = y_j | Y = y_i )$ using $\hat{\bth}$. $\bP$ is the noise distribution induced by the true parameters $\bth$ and $\bQ$ is an approximation obtained from inference with the \emph{estimated} parameters $\hat{\bth}$. 
With this terminology, we can define the perturbed version of the distance function and a corresponding replacement of \eqref{eq:f-hat}:
\begin{equation}
    \tilde{d}_q(T,\tilde{Y}=y_j)  \ldef \sum_{i=1}^k (\bQ^{-1})_{ji} d^2_{\cY} (T,Y=y_i)  \quad \forall y_j \in \cY_s,
    \label{eq:tilde-d-weak}
\end{equation}
\begin{equation}
\tilde{F}_q(x,y) \ldef \frac{1}{n}\sum_{i=1}^n \alpha_i(x)\tilde{d}_q(y,\tilde{z}_i) \qquad \hat{f}_{q}(x) = \arg\min_{y \in \cY} \tilde{F}_q(x,y).
\label{eq:f-hat-q}
\end{equation}

We similarly define $\tilde{d}_p, \tilde{F}_p, \hat{f}_p$ using the true noise distribution $\bP$. The perturbed distance $\tilde{d}_p$ is an unbiased estimator of the true distance. However we do not know the true noise distribution $\bP$ hence we cannot use it for prediction. Instead we use $\tilde{d}_q$. Note that $\tilde{d}_q$ is no longer an unbiased estimator---its bias can be expressed as function of the parameter recovery error bound in Theorem~\ref{thm:meanfinite}. 


\subsection{Bounding the Generalization Error}
What can we say about the excess risk $R(\hat{f}_{q}) - R(f^*)$? 
Note that compared to the prediction based on clean labels, there are two additional sources of error. One is the noise in the labels (i.e., even if we know the true $\bP$, the quality of the pseudolabels is imperfect). The other is our estimation procedure for the noise distribution. We must address both sources of error. 

Our analysis uses the following assumptions on the minimum and maximum singular values  $\sigma_{\min}(\bP)$ , $\sigma_{\max}(\bP)$ and the condition number $\kappa(\bP)$ of true noise matrix $\bP$ and the function $F$. Additional detail is provided in the Appendix.
\begin{restatable}{assumption}{}{(Noise model is not arbitrary) }
The true parameters $\bth$ are such that  $\sigma_{\min}(\bP) > 0$, and the condition number $\kappa(\bP)$ is sufficiently small.
\label{asm:noise-good}
\end{restatable}
\begin{restatable}{assumption}{}{(Normalized features)} $|\alpha(x)| \le 1$, for all $x \in \cX$.
\label{asm:normalized}
\end{restatable}
\begin{restatable}{assumption}{}(Proxy strong convexity)
The function $F$ in \eqref{eq:f-hat} satisfies the following property with some $\beta >0$. As we move away from the minimizer of $F$, the function increases and the rate of increase is proportional to the distance between the points:
 \begin{equation}
    F\big(x,f(x)\big) \ge F\big(x,\hat{f}(x)\big)+ \beta \cdot d_\cY^2\big(f(x),\hat{f}(x)\big) \qquad \forall x\in \cX, \forall f\in \cF.
 \end{equation}
\label{asm:strong-cvx}
\end{restatable}
\vspace{-10pt}
With these assumptions, we provide a generalization result for prediction with pseudolabels,
\begin{restatable}[]{theorem}{thmGenErrTrueNoise} (Generalization Error )
Let $\hat{f}$ be the minimizer as defined in \eqref{eq:f-hat} over the clean labels and let $\hat{f}_q$ (defined in \eqref{eq:f-hat-q}) be the minimizer over the noisy labels obtained from inference in Algorithm \ref{alg:finite}. Suppose Assumptions \ref{asm:noise-good},\ref{asm:normalized},\ref{asm:strong-cvx} hold. Then for 
$\epsilon_2 = k^{5/2}\cdot \tilde{\cO}(\epsilon(d^+) + \epsilon(d^-)) \cdot\Big(1 + \frac{\kappa(\bP)}{\sigma_{\min}(\bP)}\Big) $
    and $ c_1 = 1+\frac{\sqrt{k}}{\sigma_{\min}(\bP)}$,
     with high probability, 
\begin{equation}
    R(\hat{f}_q) \le R(f^*) + \cO(n^{-\frac{1}{4}}) + \tilde{\cO} \Big( \frac{c_1}{\beta}n^{-\frac{1}{2}}  \Big) + \tilde{\cO} \Big(\frac{3\epsilon_2}{\beta}n^{-\frac{1}{2}} \Big). 
\end{equation}
\label{thm:gen_err_true_noise}
\end{restatable}
\vspace{-10pt}
\paragraph{Implications and Tradeoffs:}
We interpret each term in the bound. The first term is present even with access to the clean labels and hence unavoidable. The second term is the additional error we incur if we learn with the knowledge of the true noise distribution. 
The third term is due to the use of the estimated noise model. It is dominated by the noise rate recovery result in Theorem~\ref{thm:meanfinite}. If the third term goes to 0 (perfect recovery) then we obtain the rate $\cO(n^{-1/4})$, the same as in the case of access to clean labels. The third term is introduced by our noise rate recovery algorithm and has two terms: one dominated by $\tilde{\cO}(n^{-1/2})$ and the other on $\tilde{\cO}(\sqrt{k}/d)$ (see discussion of Theorem \ref{thm:meanfinite}). Thus we only pay an extra additive factor $\cO(\sqrt{k}/d)$ in the excess risk when using pseudolabels. 

\section{Manifold-Valued Label Spaces: Noise Recovery and Generalization}
\label{sec:continuous}
We introduce a simple recovery method for weak supervision in constant-curvature Riemannian manifolds. First we briefly introduce some background notation on these spaces, then provide our estimator and consistency result, then the downstream generalization result. Finally, we discuss extensions to symmetric Riemannian manifolds, an even more general class of spaces.
\paragraph{Background on Riemannian manifolds} The following is necessarily a very abridged background; more detail can be found in \cite{LeeManifolds, TuManifolds}. A smooth manifold $M$ is a space where each point is located in a neighborhood diffeomorphic to $\R^d$. Attached to each point $p \in \M$ is a \emph{tangent space} $T_pM$; each such tangent space is a $d$-dimensional vector space enabling the use of calculus.

A Riemannian manifold equips a smooth manifold with a Riemannian metric: a smoothly-varying inner product $\langle \cdot, \cdot \rangle_p$ at each point $p$. This tool allows us to compute angles, lengths, and ultimately, distances $d_{\M}(p,q)$ between points on the manifold as shortest-path distances. These shortest paths are called geodesics and can be parameterized as curves $\gamma(t)$, where $\gamma(0) = p$, or by tangent vectors $v \in T_pM$. The exponential map operation $\exp: T_p\M \rightarrow \M$ takes tangent vectors to manifold points. It enables switching between these tangent vectors: $\exp_p(v) = q$ implies that $d_{\M}(p,q) = \|v\|$.


\paragraph{Invariant} Our first contribution is a simple invariant that enables us to recover the error parameters. Note that  we cannot rely on the finite metric-space technique, since the manifolds we consider have an infinite number of points. Nor do we need an embedding---we have a continuous representation as-is. Instead, we propose a simple idea based on the law of cosines. Essentially, on average, the geodesic triangle formed by the latent variable $y \in \M$ and two observed LFs $\lambda_a, \lambda_b$, is a right triangle. This means it can be characterized by the (Riemannian) version of the Pythagorean theorem:

\begin{restatable}[]{lemma}{invariantlemma}
\label{lem:invariant}
For $\cY = \M$, a hyperbolic manifold, $y \sim P$ for some distribution $P$ on $\M$ and labeling functions $\lambda_a, \lambda_b$ drawn from \eqref{eq:lbl-model-orig-a},
$\mathbb{E}\cosh d_{\cY}(\lambda_a, \lambda_b) = \mathbb{E}\cosh d_{\cY}(\lambda_b, y)\mathbb{E}\cosh d_{\cY}(\lambda_b, y)$,
while for $\cY = \M$ a spherical manifold,
$\mathbb{E}\cos d_{\cY}(\lambda_a, \lambda_b) = \mathbb{E}\cos d_{\cY}(\lambda_b, y)\mathbb{E}\cos d_{\cY}(\lambda_b, y). $
\end{restatable}

These invariants enable us to easily learn by forming a triplet system. Suppose we construct the equation in Lemma~\ref{lem:invariant} for three pairs of labeling functions. The resulting system can be solved to express $\mathbb{E}[\cosh(d_{\cY}(\lf_a,y))]$ in terms of $\mathbb{E}\cosh(d_{\cY}(\lf_a,\lf_b)), \mathbb{E}\cosh(d_{\cY}(\lf_a,\lf_c)),\mathbb{E}\cosh(d_{\cY}(\lf_b,\lf_c))$.
Specifically, 
\[{\mathbb{E}} \cosh(d_{\cY}(\lf_a, y)) = \sqrt{\frac{{\mathbb{E}}\cosh d_{\cY}(\lf_a,\lf_b) {\mathbb{E}} \cosh d_{\cY}(\lf_a, \lf_c)}{({\mathbb{E}}\cosh(d_{\cY}(\lf_b, \lf_c))^2}}. \]
Note that we can estimate $\hat{\mathbb{E}}$ via the empirical versions of terms on the right , as these are based on observable quantities. This is a generalization of the binary case in \cite{fu2020fast} and the Gaussian (Euclidean) case in \cite{Shin22} to hyperbolic manifolds. A similar estimator can be obtained for spherical manifolds by replacing $\cosh$ with $\cos$.

Using this tool, we can obtain a consistent estimator for $\theta_a$ for each of $a = 1, \ldots, m$. Let $C_0$ satisfy the following inequality $\mathbb{E}|\hat{\mathbb{E}}\cosh(d_{\cY}(\lf_a, \lf_b)) - \mathbb{E}\cosh(d_{\cY}(\lf_a, \lf_b))| \geq C_0 \mathbb{E}|\hat{\mathbb{E}}d_{\cY}^2(\lf_a, \lf_b)) - \mathbb{E}d_{\cY}^2(\lf_a, \lf_b)|$; that is, $C_0$ reflects the preservation of concentration when moving from distribution $\cosh(d)$ to $d^2$. Then,

\begin{restatable}[]{theorem}{contPE}
\label{thm:contPE}
Let $\M$ be a hyperbolic manifold. Fix $0 < \delta < 1$ and let $\Delta(\delta) = \min_{\rho} \text{Pr}\Big(\forall i, d_{\cY}(\lf_{a,i}, \lf_{b,i)}) \leq \rho \Big) \geq 1-\delta$. Then, there exists a constant $C_1$ so that with probability at least $1- \delta$,
$\mathbb{E}|\hat{\mathbb{E}}d_{\cY}^2(\lf_a, y)) - \mathbb{E}d_{\cY}^2(\lf_a, y)|\leq  {C_1 \cosh(\Delta(\delta))^{3/2}}/{C_0\sqrt{2n}}.$
\end{restatable}
As we hoped, our estimator is consistent. Note that we pay a price for a tighter bound: $\Delta(\delta)$ is large for smaller probability $\delta$. It is possible to estimate the size of $\Delta(\delta)$ (more generally, it is a function of the curvature). We provide more details in the Appendix.

Next, we adapt the downstream model predictor \eqref{eq:f-hat} in the following way. Let $\hat{\mu}_a^2 = \hat{\mathbb{E}}[d_{\cY}^2(\lf_a, y)]$. Let $\beta = [\beta_1, \ldots, \beta_m]^T$ be such that $\sum_a \beta_a = 1$ and $\beta$ minimizes $\sum_{a} \beta_a^2 \hat{\mu}_a^2$. Then, we set
\begin{align*}
    \tilde{f}(x) = \argmin_{y \in \cY} \frac{1}{n}\sum_{i=1}^n \alpha_i(x) \sum_{a=1}^m \beta_a^2 d_{\cY}^2(y,\lf_{a,i}).   
\end{align*}

We simply replace each of the true labels with a combination of the labeling functions. With this, we can state our final result. First, we introduce our assumptions.

Let $q = \argmin_{z \in \cY} \mathbb{E} [\alpha(x)(y) d_{\cY}^2(z, y)]$, where the expectation is taken over the population level distribution and $\alpha(x)(y)$ denotes the kernel at $y$.

\begin{restatable}{assumption}{}{(Bounded Hugging Function c.f. \cite{Stromme})}
Let $q$ be defined as above. For all $a,b \in \M$, the hugging function at $q$ is given by
$k_q^b(a) = 1 - (\|\log_q(a) - \log_q(b)\|^2 - d_{\cY}^2(a,b))/d_{\cY}^2(q,b)$. We assume that $k_q^b(a)$ is lower bounded by $k_{\min}$.
\label{ass:bhf}
\end{restatable}
\begin{restatable}{assumption}{}{(Kernel Symmetry)} We assume that for all $x$ and all $v \in T_q \M$, $\alpha(x)(\exp_q(v)) = \alpha(x)(\exp_q(-v))$.
\label{ass:sym}
\end{restatable}
The first condition provides control on how geodesic triangles behave; it relates to the curvature. We provide more details on this in the Appendix. The second assumption restricts us to kernels symmetric about the minimizers of the objective $F$. Finally, suppose we draw $(x,y)$ and $(x',y')$ independently from $P_{XY}$. Set $\sigma^2_o = \alpha(x)(y) \mathbb{E}d_{\cY}^2(y, y')$. 

\begin{restatable}[]{theorem}{contGen}
\label{thm:contGen}
Let $\M$ be a complete manifold and suppose the assumptions above hold. Then, there exist constants $C_3$, $C_4$
\[\mathbb{E}[d_{\cY}^2(\hat{f}(x), \tilde{f}(x))] \leq \frac{C_3\sigma_o^2}{n k_{\min}} + \frac{C_4\sum_{a=1}^m \beta^2_a \hat{\mu}_a^2}{mn k_{\min}}.\]
\end{restatable}

\begin{figure}[h]
    \centering
    \includegraphics[width=.4\linewidth]{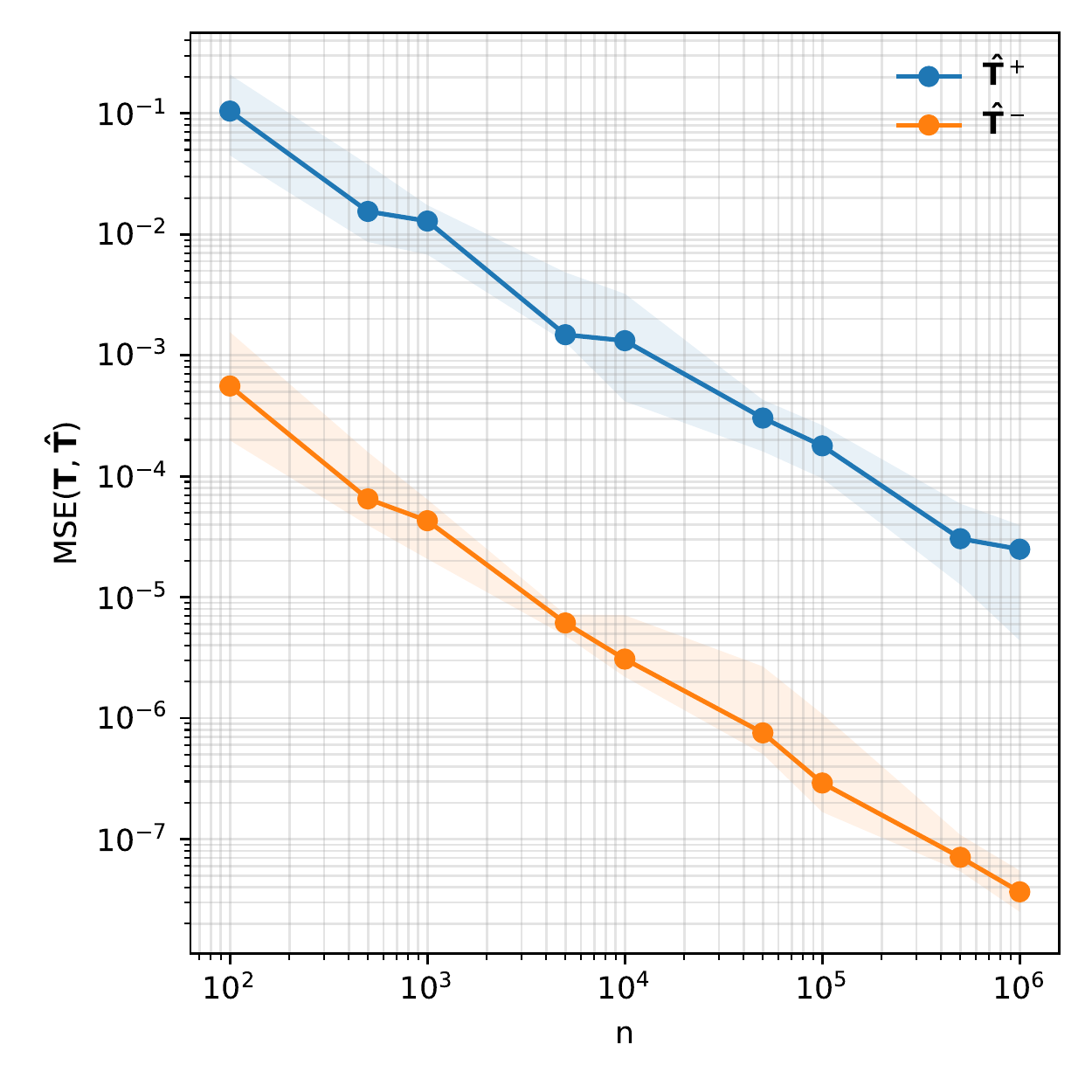}
    \includegraphics[width=.4\linewidth]{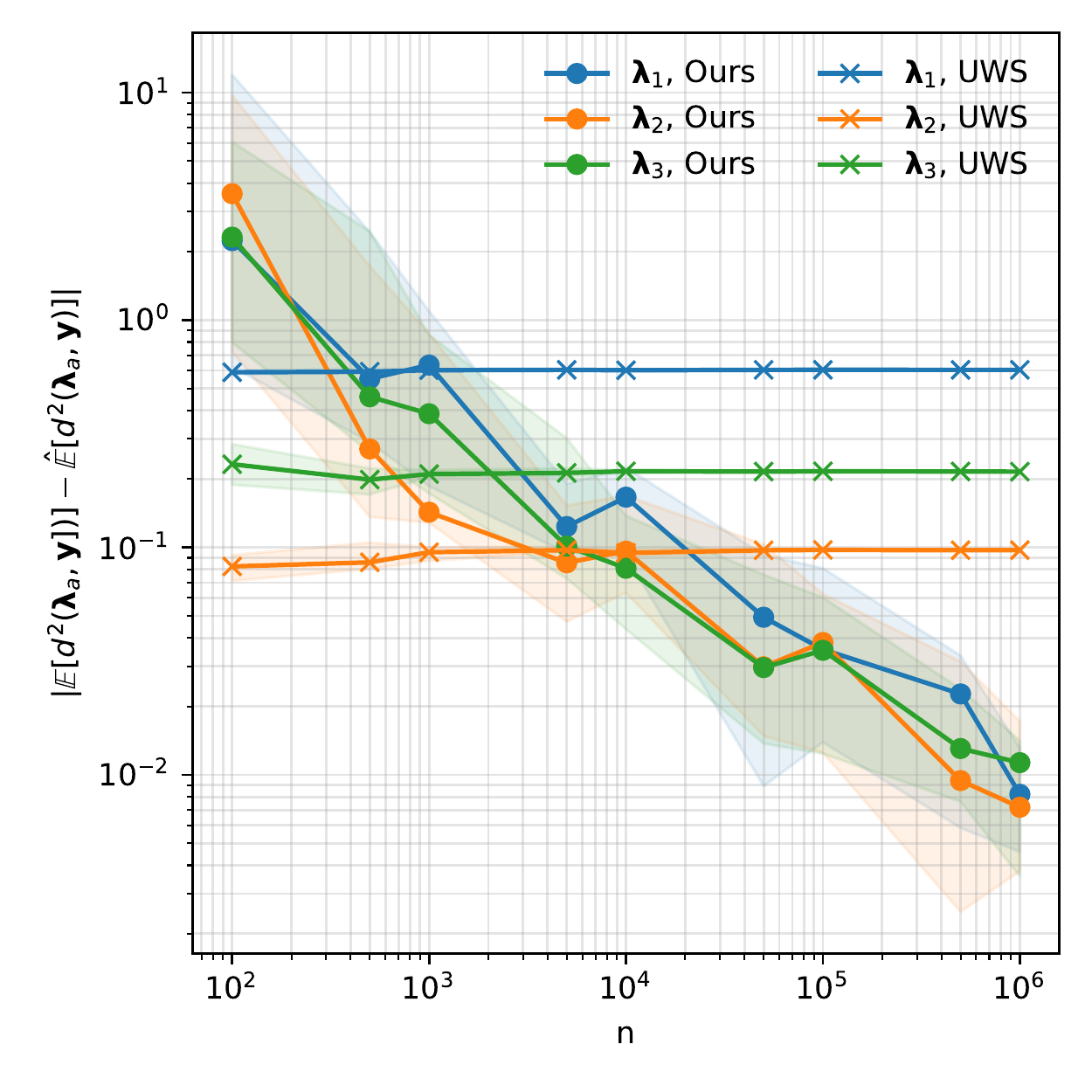}\\
    \includegraphics[width=0.8\linewidth]{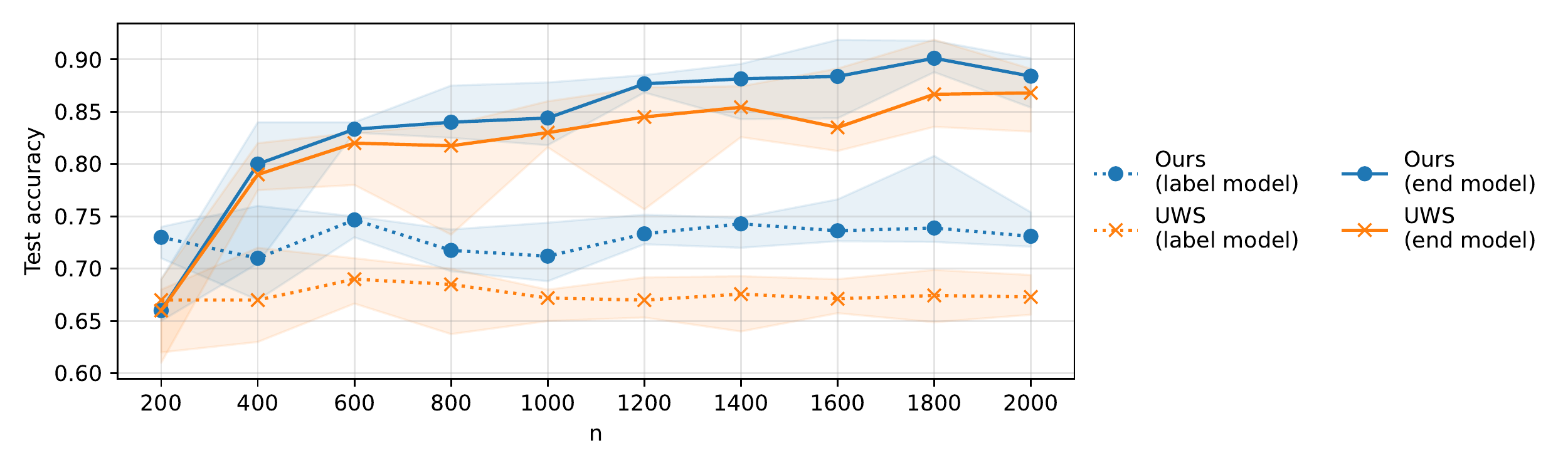}
    \caption{ {\bf Finite metric space case.} Parameter estimation improves with samples $n$ in learning to rank---showing nearly-consistent behavior. Our tensor decomposition estimator outperforms \cite{Shin22}. In particular, (top left) as the number of samples increases, our estimates of the positive and negative components of $\mathbf{T}$ improve. (Top right) the improvements in $\mathbf{T}$ recovery with more samples translates to significantly improved performance over \cite{Shin22}, which is close to constant across $n$. (Bottom) this improved parameter estimation further translates to improvements in label model accuracy (using only the noisy estimates for prediction, without training an end model) and end model generalization. For the top two plots, we use $\boldsymbol{\theta} = [6, 3, 8]$, and in the bottom plot, we use $\boldsymbol{\theta} = [0, 0, 1]$. In all plots, we report medians along with upper and lower quartiles across 10 trials. }
    \label{fig:experiments_ranking}
\end{figure}

Note that as both $m$ and $n$ grow, as long as our worst-quality LF has bounded variance, our estimator of the true predictor is consistent. Moreover, we also have favorable dependence on the noise rate. This is because the only error we incur is in computing sub-optimal $\beta$ coefficients. We comment on this suboptimality in the Appendix.

A simple corollary of Theorem~\ref{thm:contGen} provides the generalization guarantees we sought,
\begin{restatable}[]{corollary}{contGenCorr}
\label{corr:contGenCorr}
Let $\M$ be a complete manifold and suppose the assumptions above hold. Then, with high probability, $ R(\tilde{f}) \le R(f^*) + \cO(n^{-\frac{1}{4}}). $
\end{restatable}

\paragraph{Extensions to Other Manifolds}
First, we note that all of our approaches almost immediately lift to products of constant-curvature spaces. For example, we have that $\M_1 \times \M_2$ has metric $d_{\cY}^2(p,q) = d^2_{\M_1}(p_1, q_1) + d^2_{\M_2}(p_2, q_2)$, where $p_i, q_i$ are the projections of $p,q$ onto the $i$th component. 

We can go beyond products of constant-curvature spaces as well. To do so, we can build generalizations of the law of cosines (as needed for the invariance in Lemma~\ref{lem:invariant}). For example, it is possible to do so for symmetric Riemannian manifolds using the tools in \cite{Asklaksen91}.

\section{Experiments}
Finally, we validate our theoretical claims with experimental results demonstrating improved parameter recovery and end model generalization using our techniques over that of prior work \cite{Shin22}. We illustrate both the finite metric space and continuous space cases by targeting rankings (i.e., permutations) and hyperbolic spaces. In the case of rankings we show that our pseudo-Euclidean embeddings with tensor decomposition estimator yields stronger parameter recovery and downstream generalization than \cite{Shin22}. In the case of hyperbolic regression (an example of a Riemannian manifold), we show that our estimator yields improved parameter recovery over \cite{Shin22}. 

\paragraph{Finite metric spaces: Learning to rank}
To experimentally evaluate our tensor decomposition estimator for finite metric spaces, we consider the problem of learning to rank. We construct a synthetic dataset whose ground truth comprises $n$ samples of two distinct rankings among the finite metric space of all length-four permutations. We construct three labeling functions by sampling rankings according to a Mallows model, for which we obtain pseudo-Euclidean embeddings to use with our tensor decomposition estimator. 

In Figure~\ref{fig:experiments_ranking} (top left), we show that as we increase the number of samples, we can obtain an increasingly accurate estimate of $\mathbf{T}$---exhibiting the \emph{nearly-consistent} behavior predicted by our theoretical claims. This leads to downstream improvements in parameter estimates, which also become more accurate as $n$ increases. In contrast, we find that the estimates of the same parameters given by \cite{Shin22} do not improve substantially as $n$ increases, and are ultimately worse (see Figure~\ref{fig:experiments_ranking}, top right). Finally, this leads to improvements in the label model accuracy as compared to that of \cite{Shin22}, and translates to improved accuracy of an end model trained using synthetic samples (see Figure~\ref{fig:experiments_ranking}, bottom). 

\paragraph{Riemannian manifolds: Hyperbolic regression}
We similarly evaluate our estimator using synthetic labels from a hyperbolic manifold, matching the setting of Section~\ref{sec:continuous}. As shown in Figure~\ref{fig:experiments_hyperbolic}, we find that our estimator consistently outperforms that of \cite{Shin22}, often by an order of magnitude. 
\begin{figure}[h]
    \centering
    \includegraphics[width=.6\linewidth]{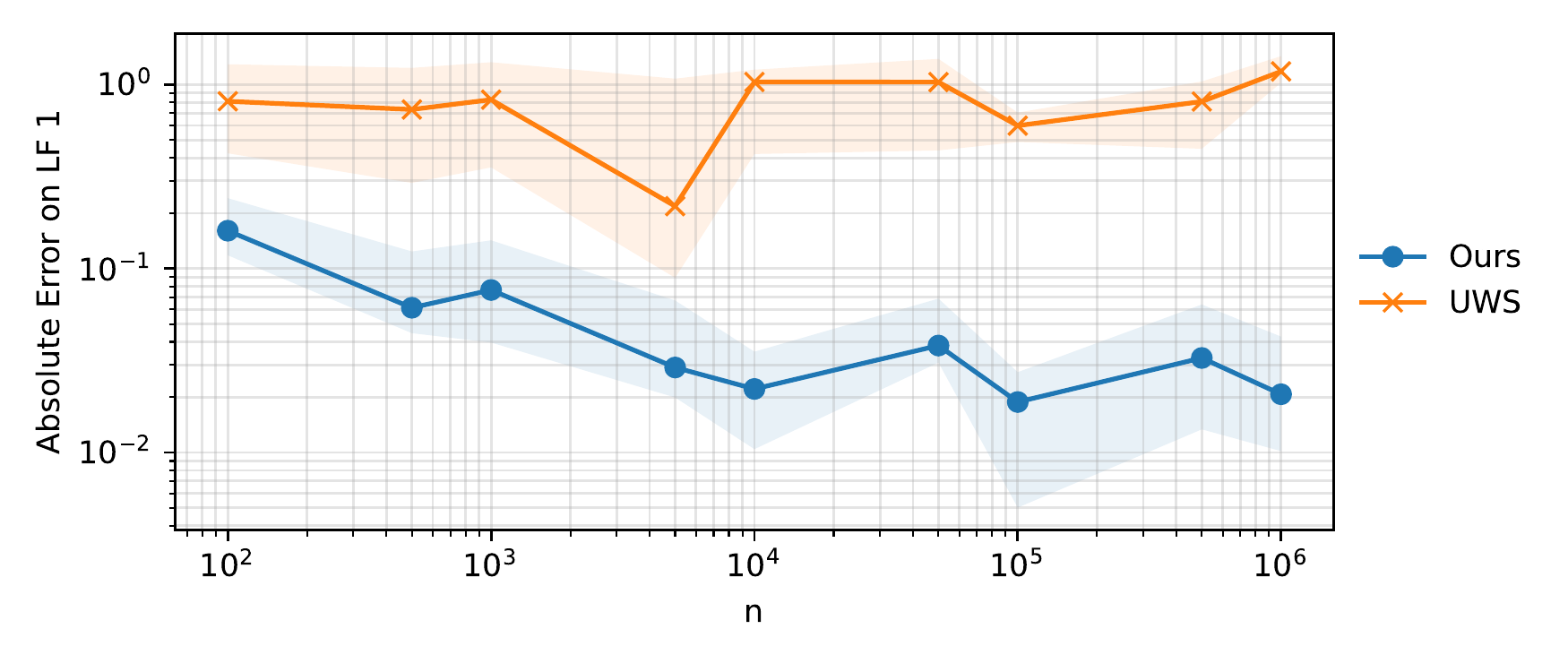}
    \caption{{\bf Continuous case.} Parameter estimation improves with more samples in the hyperbolic regression problem. Our estimator outperforms \cite{Shin22}. Here, we use different randomly sampled values of $\boldsymbol{\theta}$ for each run. We report medians along with upper and lower quartiles across 10 trials.}
    \label{fig:experiments_hyperbolic}
\end{figure}

\section{Conclusion}
We studied the theoretical properties of weak supervision applied to structured prediction in two general scenarios: label spaces that are finite metric spaces or constant-curvature manifolds. We introduced ways to estimate the noise rates of labeling functions, achieving consistency or near-consistency. Using these tools, we established that with suitable modifications downstream structured prediction models maintain generalization guarantees. Future directions include extending these results to even more general manifolds and removing some of the assumptions needed in our analysis.
\section*{Acknowledgments} We are grateful for the support of the NSF (CCF2106707), the American Family Funding Initiative and the Wisconsin Alumni Research Foundation (WARF).

\bibliography{references}

\newcommand{\etalchar}[1]{$^{#1}$}
\begin{thebibliography}{RCMR18}

\bibitem[AGH{\etalchar{+}}14]{anandkumar2014tensor}
Animashree Anandkumar, Rong Ge, Daniel Hsu, Sham~M Kakade, and Matus Telgarsky.
\newblock Tensor decompositions for learning latent variable models.
\newblock {\em Journal of Machine Learning Research}, 15:2773--2832, 2014.

\bibitem[AGJ14]{anandkumar2014sample}
Animashree Anandkumar, Rong Ge, and Majid Janzamin.
\newblock Sample complexity analysis for learning overcomplete latent variable
  models through tensor methods.
\newblock {\em arXiv preprint arXiv:1408.0553}, 2014.

\bibitem[AH91]{Asklaksen91}
Helmer Aslaksen and Hsueh-Ling Huynh.
\newblock {Laws of trigonometry in symmetric spaces}.
\newblock {\em Geometry from the Pacific Rim}, 1991.

\bibitem[BDDH14]{baraniuk2014introduction}
Richard Baraniuk, Mark~A Davenport, Marco~F Duarte, and Chinmay Hegde.
\newblock An introduction to compressive sensing.
\newblock 2014.

\bibitem[BHS{\etalchar{+}}07]{Bakir}
G\"{u}khan~H. Bakir, Thomas Hofmann, Bernhard Sch\"{o}lkopf, Alexander~J.
  Smola, Ben Taskar, and S.~V.~N. Vishwanathan.
\newblock {\em Predicting Structured Data (Neural Information Processing)}.
\newblock The MIT Press, 2007.

\bibitem[BS16]{Bapat16}
R.B. Bapat and Sivaramakrishnan Sivasubramanian.
\newblock Squared distance matrix of a tree: Inverse and inertia.
\newblock {\em Linear Algebra and its Applications}, 491:328--342, 2016.

\bibitem[CRR16]{Ciliberto2016}
Carlo Ciliberto, Lorenzo Rosasco, and Alessandro Rudi.
\newblock A consistent regularization approach for structured prediction.
\newblock In {\em Advances in Neural Information Processing Systems 30 (NIPS
  2016)}, volume~30, 2016.

\bibitem[Dem92]{james92MatPrtb}
James Demmel.
\newblock The component-wise distance to the nearest singular matrix.
\newblock {\em SIAM Journal on Matrix Analysis and Applications}, 13(1):10--19,
  1992.

\bibitem[DRS{\etalchar{+}}20]{Dunnmon20}
Jared~A. Dunnmon, Alexander~J. Ratner, Khaled Saab, Nishith Khandwala, Matthew
  Markert, Hersh Sagreiya, Roger Goldman, Christopher Lee-Messer, Matthew~P.
  Lungren, Daniel~L. Rubin, and Christopher R\'{e}.
\newblock Cross-modal data programming enables rapid medical machine learning.
\newblock {\em Patterns}, 1(2), 2020.

\bibitem[FCS{\etalchar{+}}20]{fu2020fast}
Daniel~Y. Fu, Mayee~F. Chen, Frederic Sala, Sarah~M. Hooper, Kayvon Fatahalian,
  and Christopher R\'e.
\newblock Fast and three-rious: Speeding up weak supervision with triplet
  methods.
\newblock In {\em Proceedings of the 37th International Conference on Machine
  Learning (ICML 2020)}, 2020.

\bibitem[Gol85]{goldfarb85psuedo}
Lev Goldfarb.
\newblock A new approach to pattern recognition.
\newblock {\em Progress in Pattern Recognition 2}, pages 241--402, 1985.

\bibitem[GS19]{Graber19}
Colin Graber and Alexander Schwing.
\newblock Graph structured prediction energy networks.
\newblock In {\em Advances in Neural Information Processing Systems 33 (NeurIPS
  2019)}, volume~33, 2019.

\bibitem[KAG18]{Korba18}
Anna Korba and Florence d’Alch\'{e}-Buc Alexandre~Garcia.
\newblock A structured prediction approach for label ranking.
\newblock In {\em Advances in Neural Information Processing Systems 32 (NeurIPS
  2018)}, volume~32, 2018.

\bibitem[KL15]{Kuleshov15}
Volodymyr Kuleshov and Percy~S Liang.
\newblock Calibrated structured prediction.
\newblock In {\em Advances in Neural Information Processing Systems 28 (NIPS
  2015)}, 2015.

\bibitem[KW78]{KruskalWish1978}
J.B. Kruskal and M.~Wish.
\newblock {\em {Multidimensional Scaling}}.
\newblock Sage Publications, 1978.

\bibitem[Lee00]{LeeManifolds}
John~M. Lee.
\newblock {\em Introduction to Smooth Manifolds}.
\newblock Springer, 2000.

\bibitem[LRBM06]{laub2006information}
Julian Laub, Volker Roth, Joachim~M Buhmann, and Klaus-Robert M{\"u}ller.
\newblock On the information and representation of non-euclidean pairwise data.
\newblock {\em Pattern Recognition}, 39(10):1815--1826, 2006.

\bibitem[NDRT13]{natarajan2013Noise}
Nagarajan Natarajan, Inderjit~S. Dhillon, Pradeep Ravikumar, and Ambuj Tewari.
\newblock Learning with noisy labels.
\newblock In {\em Proceedings of the 26th International Conference on Neural
  Information Processing Systems - Volume 1}, NIPS'13, page 1196–1204, 2013.

\bibitem[PHD{\etalchar{+}}06]{Pekalska06}
El{\.{z}}bieta P{\k{e}}kalska, Artsiom Harol, Robert P.~W. Duin, Barbara
  Spillmann, and Horst Bunke.
\newblock Non-euclidean or non-metric measures can be informative.
\newblock In Dit-Yan Yeung, James~T. Kwok, Ana Fred, Fabio Roli, and Dick
  de~Ridder, editors, {\em Structural, Syntactic, and Statistical Pattern
  Recognition}, pages 871--880, 2006.

\bibitem[PM19]{Petersen19}
Alexander Petersen and Hans-Georg M{\"u}ller.
\newblock Fr\'{e}chet regression for random objects with euclidean predictors.
\newblock {\em Annals of Statistics}, 47(2):691--719, 2019.

\bibitem[PPD01]{Pekalska01}
El{\.{z}}bieta P{\k{e}}kalska, Pavel Paclik, and Robert~P.W. Duin.
\newblock A generalized kernel approach to dissimilarity-based classification.
\newblock {\em Journal of Machine Learning Research}, 2:175--211, 2001.

\bibitem[RBE{\etalchar{+}}18]{Ratner18}
Alexander Ratner, Stephen~H. Bach, Henry Ehrenberg, Jason Fries, Sen Wu, and
  Christopher R\'{e}.
\newblock Snorkel: Rapid training data creation with weak supervision.
\newblock In {\em Proceedings of the 44th International Conference on Very
  Large Data Bases (VLDB)}, Rio de Janeiro, Brazil, 2018.

\bibitem[RCMR18]{Rudi18}
Alessandro Rudi, Carlo Ciliberto, GianMaria Marconi, and Lorenzo Rosasco.
\newblock Manifold structured prediction.
\newblock In {\em Advances in Neural Information Processing Systems 32 (NeurIPS
  2018)}, volume~32, 2018.

\bibitem[RHD{\etalchar{+}}19]{Ratner19}
A.~J. Ratner, B.~Hancock, J.~Dunnmon, F.~Sala, S.~Pandey, and C.~R\'{e}.
\newblock Training complex models with multi-task weak supervision.
\newblock In {\em Proceedings of the AAAI Conference on Artificial
  Intelligence}, Honolulu, Hawaii, 2019.

\bibitem[RNGS20]{re2019overton}
Christopher R{\'e}, Feng Niu, Pallavi Gudipati, and Charles Srisuwananukorn.
\newblock Overton: A data system for monitoring and improving machine-learned
  products.
\newblock In {\em Proceedings of the 10th Annual Conference on Innovative Data
  Systems Research}, 2020.

\bibitem[RSW{\etalchar{+}}16]{Ratner16}
A.~J. Ratner, Christopher M.~De Sa, Sen Wu, Daniel Selsam, and C.~R\'{e}.
\newblock Data programming: Creating large training sets, quickly.
\newblock In {\em Proceedings of the 29th Conference on Neural Information
  Processing Systems (NIPS 2016)}, Barcelona, Spain, 2016.

\bibitem[SLB20]{Bach20}
Esteban Safranchik, Shiying Luo, and Stephen Bach.
\newblock Weakly supervised sequence tagging from noisy rules.
\newblock In {\em Proceedings of the AAAI Conference on Artificial Intelligence
  (AAAI)}, pages 5570--5578, Apr. 2020.

\bibitem[SLV{\etalchar{+}}22]{Shin22}
Changho Shin, Winfred Li, Harit Vishwakarma, Nicholas~Carl Roberts, and
  Frederic Sala.
\newblock Universalizing weak supervision.
\newblock In {\em International Conference on Learning Representations}, 2022.

\bibitem[Str20]{Stromme}
Austin~J. Stromme.
\newblock {\em Wasserstein Barycenters: Statistics and Optimization}.
\newblock MIT, 2020.

\bibitem[Tu11]{TuManifolds}
Loring~W. Tu.
\newblock {\em An Introduction to Manifolds}.
\newblock Springer, 2011.

\bibitem[vRW18]{rooyen18Noise}
Brendan van Rooyen and Robert~C. Williamson.
\newblock A theory of learning with corrupted labels.
\newblock {\em Journal of Machine Learning Research}, 18(228):1--50, 2018.

\bibitem[ZS16]{Zhang16}
Hongyi Zhang and Suvrit Sra.
\newblock First-order methods for geodesically convex optimization.
\newblock In {\em Conference on Learning Theory, {COLT} 2016}, 2016.

\end{thebibliography}
\bibliographystyle{alpha}


\appendix
\section*{Appendix}
The Appendix is organized as follows. First, we provide a glossary that summarizes the notation we use throughout the paper. Afterwards, we provide the proofs for the finite-valued metric space cases. We continue with the proofs and additional discussion for the manifold-valued label spaces. Finally, we give some additional explanations for pseudo-Euclidean spaces. 

\section{Glossary} \label{appendix:glossary}
\label{sec:gloss}
The glossary is given in Table~\ref{table:glossary} below.
\begin{table*}[h]
\centering
\begin{tabular}{l l}
\toprule
Symbol & Definition \\
\midrule
$\mathcal{X}$ & feature space \\
$\mathcal{Y}$ & label metric space\\
$\cY_s$ & support of prior distribution on true labels\\ 
$d_{\mathcal{Y}}$ & label metric (distance) function \\
$x_1,x_2, \ldots, x_n$ & unlabeled datapoints from $\mathcal{X}$\\
$y_1, y_2, \ldots, y_n$ & latent (unobserved) labels from $\mathcal{Y}$\\
$s_1, s_2, \ldots, s_m$ & labeling functions / sources \\
$\lambda_1, \lambda_2, \ldots, \lambda_m$ & output of labeling functions (LFs) \\
$\blm_1, \blm_2, \ldots, \blm_m$ & pseudo-Euclidean embeddings of LFs outputs\\
$\lambda_{a,i}$ & output of $a$th LF on $i$th data point $x_i$ \\
$\blm_{a,i}$ & pseudo-Euclidean embedding of output of $a$th LF on $i$th data point $x_i$ \\
$n$ & number of data points \\
$m$ & number of LFs \\
$k$ & size of the support of prior on $\cY$ i.e. $k=|\cY_s|$ \\
$r$ & size of $\cY$ for the finite case \\
$\theta_a,\hat{\theta}_a$ & true and estimated canonical parameters of model in  \eqref{eq:lbl-model-orig-a}\\
$\bth,\hat{\bth} $ & true and estimated canonical parameters arranged as vectors\\
$\E[d^2_\cY(\lambda_a, y)] $ & mean parameters in \eqref{eq:lbl-model-orig-a}\\
$g$ & pseudo-Euclidean embedding mapping \\
$\bP$ &  true noise model $P_{ij} = P_{\bth}(\tilde{Y}= y_j | Y=y_i )$ with true parameters $\bth$ \\ 
$\bQ$ & estimated noise model with parameters $\hat{\bth}$,  $Q_{ij} = P_{\hat{\bth}}(\tilde{Y}= y_j | Y=y_i )$\\
$\Lambda$ & a random element in $\cY^m$ the $m$-fold Cartesian product of $\cY$  \\
$\Lambda^{(u)}$ & $u$th element in $\cY^m$\\
$\bmu_{a,y}^+,\bmu_{a,y}^-$ & means of distributions in \eqref{eq:proxy-lbl-model} corresponding to $\R^{d^+},\R^{d^-}$  \\
$\epsilon(d^+),\epsilon(d^-)$ & error in recovering the mean parameters \eqref{eq:epsilon-plus}\\
$\sigma$ & proxy noise variance in \eqref{eq:proxy-lbl-model} \\
$F(x,y)$ & the score function in \eqref{eq:f-hat} with true labels\\
$\tilde{F}_p(x,y), \tilde{F}_q(x,y)$ & the score function in \eqref{eq:f-hat-q} with noisy labels from distributions $\bP$ and $\bQ$\\
$\hat{f}$ & minimizer of $F$ defined in \eqref{eq:f-hat}\\
$\hat{f}_p,\hat{f}_q$ & minimizers of $\tilde{F}_p,\tilde{F}_q$ as defined in \eqref{eq:f-hat}\\
$\sigma_{\max}(\bP)$ & maximum singular value of $\bP$\\ 
$\sigma_{\min}(\bP)$ & minimum singular value of $\bP$\\ 
$\kappa(\bP)$ & the condition number of matrix $\bP$\\
$\omega(\bu,\bv)  $ & angle between vectors $\bu,\bv \in \R^d$ \\

\toprule
\end{tabular}
\caption{
	Glossary of variables and symbols used in this paper.
}
\label{table:glossary}
\end{table*}



\section{Proofs for Parameter Estimation Error in Discrete Spaces}

We introduce results leading to the proofs of the theorems for the finite-valued metric space case.

 \begin{restatable}[]{lemma}{tensorResult}(\cite{anandkumar2014sample})
Let $\hat{\bT}^+,\hat{\bT}^-$ be the third order observed moments for mutually independent labeling functions triplet, as defined in \eqref{eq:T-hat-split} using a sufficiently large number $n$ of i.i.d observations drawn from models in equation \eqref{eq:proxy-lbl-model}. Suppose there are sufficiently many such triplets to cover all labeling functions. Let $\hat{\bmu}_{a,y}^+,\hat{\bmu}_{a,y}^-$ be the estimated parameters returned by the algorithm \ref{alg:finite} for all $a \in [m]$.  Let $\epsilon(d)$ be defined as above in equation \eqref{eq:epsilon-plus}, then the following holds with high probability for all labeling functions,
 \begin{equation}
    ||\bmu_{a,y}^+ - \hat{\bmu}_{a,y}^+||_2 \le \cO(\epsilon(d^+))  \quad \text{ and } \quad   ||\bmu_{a,y}^- - \hat{\bmu}_{a,y}^-||_2 \le \cO(\epsilon(d^-)) \quad \forall a \in [m] \,\, \forall y \in \cY_s 
 \end{equation}
\label{lem:tensor-decomp}
\end{restatable}
\begin{proof}
The result follows by first showing that our setting and assumptions imply that the conditions of Theorems 1 and 5 in \cite{anandkumar2014sample} are satisfied, which allows us to adopt their results. We then translate the result in order to state it in terms of the $\ell_2$ distance.

The tensor concentration result in Theorem 1 in \cite{anandkumar2014sample} relies heavily on the noise matrices satisfying the Restricted Isometry Property (RIP) property. The authors make an explicit assumption that the noise model satisfies this condition. In our setting, we have a specific form of the noise model that allows us to show that this assumption is satisfied. The RIP condition is satisfied for sub-Gaussian noise matrices \cite{baraniuk2014introduction}. Our noise matrices are supported on a discrete space and have bounded entries, and so are sub-Gaussian. 

The other required conditions on the norms of factor matrices and the number of latent factors are implied by Assumption~\ref{asm:support}. Thus, we can adopt the results on recovery of parameters $\bmu_{a,y}$ and the prior weights $w_y$ from \cite{anandkumar2014sample}. The result gives us for all $a \in [m], y\in \cY_s$,
$$\text{dist}(\bmu_{a,y}^+, \hat{\bmu}_{a,y}^+) \le \cO\big(\epsilon(d^+)\big) ,\quad \text{dist}(\bmu_{a,y}^-, \hat{\bmu}_{a,y}^-) \le \cO\big(\epsilon(d^-)\big), $$
and  $$\quad |w_y -\hat{w}_y| \le \cO\Big(\max\big(\epsilon(d^+),\epsilon(d^-)\big)/k\Big),$$
where $\text{dist}(\bu,\bv)$ is defined as follows. For any $\bu,
\bv \in \R^d$,
 $$\text{dist}(\bu,\bv) = \sup_{\bz\perp \bu} \frac{\langle \bz,\bv\rangle}{||\bz||_2||\bv||_2} = \sup_{\bz\perp \bv} \frac{\langle \bz,\bu\rangle}{||\bz||_2||\bu||_2} .$$
 
 Next, we translate the result to the Euclidean distance.  For $\bu,\bv \in \R^d$ with $||\bu||, ||\bv||=1$, it is easy to see that
 $$\min_{z \in \{-1,+1\}}||z\bu -\bv||_2 \le \sqrt{2}\,\text{dist}(\bu,\bv) .$$
   This notion of distance is oblivious to sign recovery.  However, when sign recovery is possible then the Euclidean distance can be bounded as follows,
 $$||\bu -\bv||_2 \le \sqrt{2}\,\text{dist}(\bu,\bv) .$$

Next we make use of Assumption \ref{ass:angle-bounded} to recover the signs of $\bmu^+,\bmu^-$. The assumption bounds the angle between true $\bmu_{a,y}^+$ and $\by^+$ between $[0,\pi/2-c)$ for some sufficiently small $c \in (0,\pi/4] $ such that $\sin(c) > \max(\epsilon_0(d^+),\epsilon_0(d^-))$, where $\epsilon_0(d)$ is defined for some $n_0 < n$ samples in equation \eqref{eq:epsilon-plus}. 
We measure $\omega(\hat{\bmu}_{a,y}^+,\by^+)$ and $\omega(-\hat{\bmu}_{a,y}^+,\by^+)$ and claim that whichever makes an acute angle with $\by^+$ has the correct sign.

We have that $\omega(\hat{\bmu}_{a,y}^+,\by^+) \le \omega(\hat{\bmu}_{a,y}^+,\bmu_{a,y}^+) + \omega(\bmu_{a,y}^+,\by^+) .$ Let $s \in \{-1,+1\}$ be the correct sign, then, 
\begin{align*}\omega(s\hat{\bmu}_{a,y}^+,\by^+) &\le \omega(s\hat{\bmu}_{a,y}^+,\bmu_{a,y}^+) + \omega(s\bmu_{a,y}^+,\by^+) \\ 
&\le \sin^{-1}(\epsilon(d^+)) + \pi/2 - c \\
&< \pi/2 - \Big(\sin^{-1}\big(\max(\epsilon_0(d^+),\epsilon_0(d^-)) \big)  -\sin^{-1}\big(\epsilon(d^+)\big)\Big) \\
&< \pi/2 \quad \text{since } \sin^{-1} \text{ is an increasing function in the domain under consideration.} 
\end{align*} 
With the correct sign $\sin^{-1}(\epsilon(d^+)) < \pi/2$ and so is $\omega(s\hat{\bmu}^{+}_{a,y },\by^+)$.  Thus with incorrect sign $\omega(-s\hat{\bmu}_{a,y}^+, \by^+) > \pi/2$.
 


 
 Hence, after disambiguating the signs we have,
 $$  ||\bmu_{a,y}^+ - \hat{\bmu}_{a,y}^+||_2 \le \cO(\text{dist}(\bmu_{a,y}^+,\bmu_{a,y}^-)) \le \cO(\epsilon(d^+)) $$
 and similarly for $\bmu^-_{a,y}$. Next with $n,d$ sufficiently large such that $\epsilon(d^+),\epsilon(d^-)\le 1$, the result holds for squared distances.
\end{proof}

\meanParamRecovery*

\begin{proof}

We prove this by using the bounds on errors in the estimates of  $\bmu^+_{a,y}$ and $\bmu^-_{a,y}$ from Lemma \ref{lem:tensor-decomp}. We proceed by bounding the errors in two parts for $\E[d^2_\phi(\blm_a^+,\by^+)]$ and $ \E[d^2_\phi(\blm_a^-,\by^-)]$ separately and then combine them to get the bound on overall parameter estimation error.

We first bound the error for $\E[d^2_\phi(\blm_a^+,\by^+)]$. The true mean parameter (i.e., the true expected squared distance) can be expanded as follows:
\begin{align*}
   \E[d^2_\phi(\blm_a^+,\by^+)] &= \E \Big [  ||\blm_a^+||_2^2 + ||\by^{+}||_2^2   - 2 \langle \blm_a^+,\by^{+}  \rangle \Big ],   \\
   &= \E_{\blm}[ ||\blm_a^+||_2^2] + \E_{\by}[||\by^+||_2^2] - 2\E_{\by}[ \langle \bmu_{a,y}^{+} , \by^+\rangle ].
\end{align*}
The estimate $\hat{\E}_{\blm}[||\blm_a^+||_2^2]$ is computed empirically. The first term is estimated observed LF outputs, i.e. $\hat{\E}_{\blm}[||\blm_a^+||_2^2] = \frac{1}{n}\sum_{i=1}^n || \blm_a^{(i),+}||_2^2 $. The second term is computed by using the estimated prior on the labels and for the last term we plug in the estimate of $\bmu_{a,y}^+$ computed using the tensor-decomposition algorithm. Putting them all together we have the following estimator:
\begin{align*}
   \hat{\E}[d^2_\phi(\blm_a^+,\by^+)] 
   &= \hat{\E}_{\blm}[||\blm_a^+||_2^2]  + \hat{\E}_{\by}[||\by^+||_2^2] - 2\hat{\E}_{\by} \langle \hat{\bmu}_{a,y}^+ , \by^{+}\rangle . 
\end{align*}

We want to bound the error of our estimator i.e. the  difference $|\E[d^2_\phi(\blm_a^+,\by)] -  \hat{\E}[d^2_\phi(\blm_a^+,\by)]|$. For this first consider the following,
\begin{align*}
    |\E_\by\langle \bmu_{a,y}^+ , \by^{+}\rangle - \hat{\E}_{\by} \langle \hat{\bmu}_{a,y}^+ , \by^{+}\rangle | &=  \sum_{y} \Big |\big \langle \big ( w_y \bmu^{+}_{a,y} - \hat{w}_y \hat{\bmu}^{+}_{a,y} \big) , \by^+ \big\rangle \Big | \\ 
    &= \sum_{y} \Big |\big \langle \big ( w_y \bmu^{+}_{a,y} - w_y \hat{\bmu}^{+}_{a,y} + w_y \hat{\bmu}^{+}_{a,y}  - \hat{w}_y \hat{\bmu}^{+}_{a,y} \big) , \by^+ \big\rangle \Big | \\
    &\le  \sum_{y} \big | w_y \big\langle  \big (  \bmu^{+}_{a,y} -  \hat{\bmu}^{+}_{a,y} \big) , \by^+ \big\rangle \big|+ \sum_{y} \cO(\epsilon(d^+)/k) \big|\langle \hat{\bmu}_{a,y}^{+}, \by^+ \rangle \big|   \\ 
    &\le \sum_{y} \big| w_y \big\langle  \big( \bmu^{+}_{a,y} -  \hat{\bmu}^{+}_{a,y} \big) , \by^+ \big\rangle \big| + \cO(\epsilon(d^+)) \\
    &\le \sum_{y}  w_y  || \bmu^{+}_{a,y} -  \hat{\bmu}^{+}_{a,y}||_2 ||\by^+||_2 + \cO(\epsilon(d^+)) \\
    &\le \cO(\epsilon(d^+)).
\end{align*}

Here we used $|| \bmu_{a,y}^+-\hat{\bmu}_{a,y}^+||_2 \le \cO(\epsilon(d^+))$ and $||\bmu_{a,y}^+||_2,||\hat{\bmu}_{a,y}^+||_2= 1, ||\by^+||_2 \le 1, ||\blm_a^+||_2^2 \le 1 $ and $|w_y -\hat{w}_y| \le \cO(\epsilon(d^+))/k$. 
Similarly, 
\begin{align}
\big | \hat{\E}_{\by}[||\by^+||_2^2] - \E_{\by}[||\by^+||_2^2] \big | = \big| \sum_{y} w_y ||\by^+||_2^2 - \hat{w}_y||\by^+||_2^2 \big | \le \sum_y |w_y-\hat{w}_y| \cdot ||\by^+||_2^2 \le \cO(\epsilon(d^+))
\end{align}

Hence the parameter estimator error,
\begin{align*}
    \Big |\E[d^2_\phi(\blm_a^+,\by)] -  \hat{\E}[d^2_\phi(\blm_a^+,\by)]   \Big | &\le \Big| \E_{\blm}[||\blm_a^+||_2^2 - \hat{\E}_{\blm}[||\blm_a^+||_2^2]  \Big | +  2|\E_\by\langle \bmu_{a,y}^+ , \by^{+}\rangle - \hat{\E}_{\by} \langle \hat{\bmu}_{a,y}^+ , \by^{+}\rangle | \\ & \hspace{5cm} + \big | \hat{\E}_{\by}[||\by^+||_2^2] - \E_{\by}[||\by^+||_2^2] \big |  \\
    &\le \cO(1/\sqrt{n}) +  \cO(\epsilon(d^+)) + \cO(\epsilon(d^+))  \\
    &\le \cO(\epsilon(d^+)).
\end{align*}
In the second step, we bound the first term by $\cO(1/\sqrt{n})$ via standard concentration inequalities.


Doing the same calculations for $\blm_a^-$, we obtain
\begin{align*}
      \Big |\E[d^2_\phi(\blm_a^-,\by)] -  \hat{\E}[d^2_\phi(\blm_a^-,\by)]   \Big | \le \cO(\epsilon(d^-)).
\end{align*}
The overall error in mean parameters is then
\begin{align*}
\Big|\E[d^2_\phi(\blm_a,\by)] -  \hat{\E}[d^2_\phi(\blm_a,\by)] \Big | &\le  \Big |\E[d^2_\phi(\blm_a^+,\by)] -  \hat{\E}[d^2_\phi(\blm_a^+,\by)]   \Big | + \\ & \qquad\Big |\E[d^2_\phi(\blm_a^-,\by)] -  \hat{\E}[d^2_\phi(\blm_a^-,\by)]   \Big |, \\
&\le \cO(\epsilon(d^+)) + \cO(\epsilon(d^-)) .
\end{align*}

Next, we use a known relation between the mean and the canonical parameters of the exponential model to get the result in terms of the canonical parameters:
$$|\theta_a -\hat{\theta}_a| \le \frac{1}{e_{\min}(A_a(\theta))} \big|\E[d^2_{\cY}(\lambda_a,y) - \hat{\E}[d^2_{\cY}(\lambda_a,y) ]\big| . $$
where $A_a{(\theta)}$ is the log partition function of the label model in \eqref{eq:lbl-model-orig-a} and  $e_{\min}(A_a) = \inf_{\theta \in \Theta} \frac{d^2}{d\theta^2}A_a(\theta)$ over the parameter space $\Theta$. For more details see Lemma 8 from \cite{fu2020fast} and Theorem 4.3 in \cite{Shin22}.
Letting $C_0 = \max_{a \in [m]} e_{\min}(A_a)$ concludes the proof.
\end{proof}

\section{Proofs for Generalization Error in Discrete Space}
In this section we give the proof for the generalization error bound in the discrete label spaces. We first show that the perturbed (noise-aware) distance function $\tilde{d_p}$ is an unbiased estimator of the true distance. Using this we show that the noise aware score function $\tilde{F}_p$ is a good uniform approximation of the score function $F$. Then we show that the minimizer $\hat{f}_p$ of $\tilde{F}_p$ is close to the minimizer $\hat{f}$ and that this closeness depends on how well $\tilde{F}_p$ approximates $F$. Next, showing that $\tilde{F}_q$ is a good uniform approximation of $\tilde{F}_p$ using the results from previous section on parameter recovery leads to the result on generalization error of $\hat{f}_q$. 


\begin{restatable}[]{lemma}{}
Let the distribution $\tilde{Y}|Y$ be given by $\bP$ a $k\times k$ transition probability matrix with $\bP_{ij} = \P(\tilde{Y} = y_j | Y = y_i )$  and suppose $\bP$ is invertible. Let pseudo-distance $   \tilde{d}_p(T,\tilde{Y}=y_j)  \ldef \sum_{i=1}^k (\bP^{-1})_{ji} d^2_{\cY} (T,Y=y_i)  \quad \forall y_j \in \cY_s,
$ then,
 \begin{equation}
     \E_{\tilde{Y}|Y=y_i}\big[ \tilde{d}_p(T,\tilde{Y})\big] = d_{\cY}^2(T,y_i).
 \end{equation}
 \label{lemma:unbiased}
\end{restatable}
\begin{proof}
Here we adopt the same ideas as in \cite{natarajan2013Noise} to create the unbiased estimator $\tilde{d}_p$ as follows,
First we write the equations for the expectations for each $y_i$. Which gives us a system of linear equations and solving them for $\tilde{d}_p$ gives us the expression for the unbiased estimator.
\begin{align*}
   \E_{\tilde{Y}|Y=y_i} [\tilde{d}_p(T,\tilde{Y})] = P_{\theta}(\tilde{Y}=y_1 | Y=y_i) \tilde{d}_p(T,\tilde{Y}=y_1) + \ldots +  P_{\theta}(\tilde{Y}=y_j | Y=y_i)\tilde{d}_p(T,\tilde{Y}=y_j) +\ldots \\ + P_{\theta}(\tilde{Y}=y_k | Y=y_i)\tilde{d}_p(T,\tilde{Y}=y_k) = d^2_{\cY}(T,y_i) \quad \forall i\in[k]
\end{align*}

Set $\tilde{\bd}_p \in \R^k$ with $i$th entry $\tilde{\bd}_{p}[i]$ given by $\tilde{d}_p(T,\tilde{Y}=y_i)$ and similarly define $\bd$ with $\bd[i] = d^2_{\cY}(T,y_i)$. Then the above system of linear equations can be expressed as follows,
   $$\bP \tilde{\bd}_p  = \bd \implies \tilde{\bd}_p = \bP^{-1}\bd \implies \tilde{d}_p(T,\tilde{Y}=y_j) =  \sum_{i=1}^k (\bP^{-1})_{ji} d^2_{\cY} (T,Y=y_i)  \forall y_j \in \cY_s $$
\end{proof} 

Next, we show that the noisy score function $\tilde{F}_p$ concentrates around the true score function $F$ for all $x$ and $y$ with high probability.
\begin{restatable}[]{lemma}{lemFxConcentration}
    Let $F$ and $\tilde{F}_p$ be defined as in \eqref{eq:f-hat-q}  and \eqref{eq:f-hat} over $n$ i.i.d. samples. Then the following holds for any $x\in \cX,y\in \cY$ with high probability, 
    \begin{equation}
        |F(x,y) - \tilde{F}_p(x,y)| \le \tilde{\cO} \Big( \big (1+ \frac{\sqrt{k}}{\sigma_{\min}(\bP)}\big ) \sqrt{\frac{1}{n} } \Big)  \quad \forall x \in \cX , \forall y \in \cY_s,
    \end{equation}
    \label{lem:Fx-Concentration}
    where $\sigma_{\min}(\bP) $ is the minimum singular value of $\bP$.
\end{restatable}
\begin{proof}
Let $\{y_i\}_{i=1}^n$ be the true labels of points $\{x_i\}_{i=1}^n$ and let the pseudo-label for $i$th point drawn from the true noise model $\bP$ be $\tilde{y}_i$. 
  Let $\tilde{\bd}_p \in \R^{k}$ be a vector such that its $i^{th}$ entry is given as $\tilde{\bd}_{p}[i] = \tilde{d}_p(T,\tilde{Z}=y_i) $, and similarly and  $ \bd \in \R^{k}  $ with $\bd[i] = d^{2}_{\cY}(T,Y=y_i)$. Recall the definitions of the score functions $F$ and $\tilde{F}_p$ for any $x\in \cX$ and $y$ in $\cY$,
\begin{align*}
    F(x,y) \ldef \frac{1}{n}\sum_{i=1}^n \alpha_i(x)d^2_\cY(y,y_i), \qquad \tilde{F}_p(x,y) \ldef \frac{1}{n}\sum_{i=1}^n \alpha_i(x)\tilde{d}_p(y,\tilde{y}_i) .
\end{align*}
Taking their difference,
\begin{align*}
  \tilde{F}_p(x,y) - F(x,y) &= \frac{1}{n}\sum_{i=1}^n\alpha_i(x)  \Big ( \tilde{d}_p(y,\tilde{y}_i) - d^2_\cY(y,y_i) \Big ) ,\\
    &= \frac{1}{n}\sum_{i=1}^n\alpha_i(x) \xi(y,y_i,\tilde{y}_i) .
\end{align*}


Here $y,y_i$ are fixed and the randomness is over $\tilde{y}_i$, thus we can think of $\tilde{y}_i$ as random variable $\tilde{Y}_i$ and take the expectation of $\xi$ over the distribution $\bP$. From Lemma \ref{lemma:unbiased} we have $\E_{\tilde{Y} | Y = y_i }[\xi(y,y_i,\tilde{Y})] = 0$ and this implies $\E [ \tilde{F}_p(x,y) - F(x,y)] = 0 $. 

Moreover, $\alpha_i(x)\cdot\xi(y,y_i,\tilde{Y}_i)$ are independent random variables and $\alpha_i(x) \le 1$.
%
The $\xi$ are bounded as follows as long as the spectral decomposition of $\bP$ is not arbitrary, 
$$\max_{z \in \cY_s} \tilde{d}_p(y,z) = ||\tilde{\bd}_p||_\infty = ||\bP^{-1}\bd||_\infty \le  ||\bP^{-1}||_\infty||\bd||_\infty . $$
Now using the fact that $||\bd||_\infty \le 1$ and properties of matrix norms we get,
$$||\bP^{-1}||_\infty||\bd||_\infty \le  ||\bP^{-1}||_\infty \le \sqrt{k}  ||\bP^{-1}||_2 \le \frac{\sqrt{k}}{\sigma_{\min}(\bP)}. $$
 Moreover, $ \forall y,z \in \cY_s, d^2_\cY(y,z) \le 1$ which gives us the magnitude of random variables $\xi(y,z,\tilde{z})$ is upper bounded by $c_1 \ldef 1 + \frac{\sqrt{k}}{\sigma_{\min}(\bP)} \, \forall y,z,\tilde{z} \in \cY_s$.
 Thus using Hoeffding's inequality and union bound over all $y \in \cY_s$ we get,     
$$|\tilde{F}_p(x,y) - F(x,y)| \le \tilde{\cO} \Big( c_1\sqrt{\frac{1}{n} } \Big) \quad \forall y \in \cY_s, x \in \cX .$$
Note that, the statement holds for $x \in \cX$ without requiring an explicit union bound over $x$. It is because the above concentration depends only on the labels and the events that the above inequality does not hold for any distinct $x_1,x_2 \in \cX$ are the same. 
\end{proof}

Now, we show that the distance between minimizer of $\tilde{F}_p$ and $F$ is bounded.

\begin{restatable}[]{lemma}{lemFuncDistBound}
Let $\hat{f}$ be the minimizer as defined in \eqref{eq:f-hat} over the clean labels and let $\hat{f}_p$ (defined in eq. \eqref{eq:f-hat-q}) be the minimizer over the noisy labels obtained from conditional distribution $\tilde{Y}|Y$ i.e.  $\bP$ such that lemma \ref{lemma:unbiased}, \ref{lem:Fx-Concentration} hold, and let the risk function be defined as in \eqref{eq:risk}, then with high probability, 
\begin{equation}
    d^2_{\cY}\big( \hat{f}_p(x), \hat{f}(x) \big) \le \tilde{\cO} \Big( \frac{c_1}{\beta}\sqrt{\frac{1}{n} } \Big) \quad \forall x \in \cX .
\end{equation}
\label{lem:f_p_and_f_dist}
\end{restatable}
\begin{proof}

Recall the definitions,
\begin{align*}
    \hat{f}(x) = \argmin_{y \in \cY} F(x,y) \qquad  \hat{f}_p(x) = \argmin_{y \in \cY} \tilde{F}_p(x,y)
\end{align*}

Let $d^2_{\cY}(f_1,f_2) = \sup_{x \in \cX} d^2_{\cY}\big ( f_1(x),f_2(x) \big )$ and 
let $\cB(\hat{f},r) = \{f : d^2_\cY(\hat{f},f) \le r \}$ denote the ball of radius $r$ around $\hat{f}$. 

From Lemma \ref{lem:Fx-Concentration} we know for $ t =\tilde{\cO} \Big( c_1\sqrt{\frac{1}{n}} \Big) ,$
\begin{align*} F \big(x,f(x) \big) -t \le \tilde{F}_p\big(x,f(x)\big) \le F\big(x,f(x)\big) + t  \quad \forall f:\cX \mapsto \cY_s .\end{align*}
From Assumption \ref{asm:strong-cvx} we have,
\begin{align*} F\big(x,f(x)\big) \ge F \big(x,\hat{f}(x) \big) + \beta \cdot d_\cY^2(f(x),\hat{f}(x)) .\end{align*}
Combining the two we get a lower bound on $\tilde{F}_p$,
 \begin{align*}
     \tilde{F}_p(x,f(x)) \ge F \big(x,\hat{f}(x) \big) + \beta \cdot d_\cY^2(f(x),\hat{f}(x)) -t .
 \end{align*}
We want to find a sufficiently large ball around $\hat{f}$ such that the minimizer of $\tilde{F}_p$ does not lie outside this ball. To see this let $LB$ and $UB$ denote the above mentioned lower and upper bounds on $\tilde{F}_p$,
\begin{align*}
LB(\tilde{F}_p,f,x) &\ldef F \big(x,\hat{f}(x) \big) + \beta \cdot d_\cY^2(f(x),\hat{f}(x)) -t . \\
UB(\tilde{F}_p,f,x) &\ldef F \big(x,f(x) \big) + t     .
\end{align*}

For $f \in \cB(\hat{f},\frac{2t}{\beta})$ and some $f'$ such that
\begin{align*}
    UB(\tilde{F}_p,f,x) &\le LB(\tilde{F}_p,f',x) \quad \forall x , \\
 F \big(x,f(x) \big) + t  &\le F \big(x,\hat{f}(x) \big) + \beta \cdot d_\cY^2(f'(x),\hat{f}(x)) -t ,\\
 F \big(x,f(x) \big) - F \big(x,\hat{f}(x) \big) + t  &\le \beta \cdot d_\cY^2(f'(x),\hat{f}(x)) -t , \\
  \beta d^2_{\cY}(f(x),\hat{f}(x)) + t  &\le \beta \cdot d_\cY^2(f'(x),\hat{f}(x)) -t  ,\\
 d_\cY^2(f'(x),\hat{f}(x)) &\ge 2t/\beta  + d^2_{\cY}(f(x),\hat{f}(x)) .
\end{align*}


Thus considering the greatest lower bound, any $f'$ with $ d_\cY^2(f'(x),\hat{f}(x)) \ge \frac{4t}{\beta}$  cannot be the minimizer of $\tilde{F}_p$, since there exists some other $f$ with smaller distance from $\hat{f}$ that has smaller value compared to $f'$.
\end{proof}




Next we show that a good estimate of true noise matrix $\bP$ by $\bQ$ leads to $\tilde{F}_q$ being uniformly close to $\tilde{F}_p$.

\begin{restatable}[]{lemma}{}
Let $\bQ$, $\bP$ be the distributions defined in equation \eqref{eq:dist-P}, and $\tilde{d}_q(T,\tilde{Y})$ be the distance function as in \eqref{eq:tilde-d-weak}, if $\max_{ij} |\bP_{ij}-\bQ_{ij}| = \epsilon$,
\begin{equation}
    \big| \tilde{d}_{q}(y, \tilde{z}_i  ) - \tilde{d}_{p}(y,\tilde{z}_i)  \big | \le \cO \Big(k^2 \Big(\sigma_{\max}(\bP)+\frac{\kappa(\bP)}{\sigma_{\min}(\bP)} \Big)\cdot \epsilon  \Big)  \qquad \forall y \in \cY_s .
\end{equation}
\label{lemma:bias-p-q}
\end{restatable}

\begin{proof}
    Let $\tilde{\bd}_q \in \R^{k}$ be a vector such that its $i^{th}$ entry is given as $\tilde{\bd}_{q}[i] = \tilde{d}_q(T,\tilde{Z}=y_i) $, and similarly, let $\tilde{\bd}_p \in \R^{k}$ with $\tilde{\bd}_{p}[i] = \tilde{d}_{p}(T,\tilde{Y}=y_i) ,$ and  $ \bd \in \R^{k}  $ with $\bd[i] = d^{2}_{\cY}(T,Y=y_i)$. It is easy to see that $\tilde{\bd}_q = \bQ^{-1}\bd$ and $\tilde{\bd}_p = \bP^{-1}\bd$. Now consider the following expectation w.r.t $\bP$,
    \begin{align*}
        \tilde{\bd}_q - \tilde{\bd}_p =  \bQ^{-1}\bd - \bP^{-1}\bd
        =\big(\bQ^{-1} - \bP^{-1} \big) \bd .
    \end{align*}
     Let $\Delta \bP = \bP - \bQ$, and using standard matrix inversion results for small perturbations, \cite{james92MatPrtb}, and $||\bd||_\infty \le 1$ we get the following. As $\max_{ij} (\Delta \bP)_{ij}  \le \epsilon $, we have $||\Delta \bP||_{2} \le ||\Delta \bP||_{F} \le \epsilon k $
     \begin{align*}
        || \tilde{\bd}_p - \tilde{\bd}_q  ||_\infty &\le  ||(\bP+\Delta \bP)^{-1} - \bP^{-1}||_\infty ||\bd||_\infty, \\ 
        &\le  \sqrt{k}||(\bP+\Delta \bP)^{-1} - \bP^{-1}||_2 ||\bd||_\infty, \\ 
        &= \sqrt{k} \Big( \kappa(\bP)||\bP^{-1}||_2 ||\Delta \bP||_2 \Big )  + \sqrt{k}\cO(||\Delta \bP||_2^2) ,\\
        &\le  \sqrt{k} \cdot \kappa(\bP)||\bP^{-1}||_2 \cdot  \epsilon k  + \cO(\epsilon^2 k^{5/2}) ,\\
        &\le \cO \Big(k^{5/2} \Big(1+\frac{\kappa(\bP)}{\sigma_{\min}(\bP)} \Big)\cdot \epsilon  \Big) \rdef c_2.
    \end{align*}
\end{proof}

\begin{restatable}[]{lemma}{lemFxConcentrationPQ}
    For $\tilde{F}_p$ and $\tilde{F}_q$ defined in \eqref{eq:f-hat-q} w.r.t. noise distributions $\bP$ and $\bQ$ respectively, and let  $\max_{ij} |\bP_{ij}-\bQ_{ij}| \le \epsilon$ then we have w.h.p.
    \begin{equation}
        |\tilde{F}_p(x,y) - \tilde{F}_q(x,y)| \le  \tilde{\cO}\Big( (2c_1+c_2)\sqrt{\frac{1}{n} } \Big) \qquad \forall y \in \cY_s, \forall x \in \cX.
    \end{equation} 
    \label{lem:Fx-Concentration-PQ}
    with $c_2 = k^{5/2}\cdot\epsilon \cdot\Big(1 + \frac{\kappa(\bP)}{\sigma_{\min}(\bP)}\Big) $
    and $ c_1 = 1+\frac{\sqrt{k}}{\sigma_{\min}(\bP)}$,
\end{restatable}

\begin{proof}
Recall, random variables $\tilde{Y}$,$\tilde{Z}$ denote the noisy labels drawn from true and estimated noise distributions $\bP,\bQ$ respectively and $\tilde{y}_i,\tilde{z}_i$ denote their draw for data point $x_i$. Note that we do not know $\bP$ and $\tilde{y}_i$ in practice and we only know $\bQ, \tilde{z}_i$. Here we are using $\bP$ and $\tilde{y}_i$ to compare our actual estimates using samples $\tilde{z}_i$ against the estimates one could have obtained from $\tilde{y}_i$.

Recall the definitions,
\begin{align*}
   \tilde{F}_p(x,y) \ldef \frac{1}{n}\sum_{i=1}^n \alpha_i(x)\tilde{d}_p(y,\tilde{y}_i),  \qquad \tilde{F}_q(x,y) \ldef \frac{1}{n}\sum_{i=1}^n \alpha_i(x)\tilde{d}_q(y,\tilde{z}_i) .
\end{align*}
Then,
\begin{align*}
    \tilde{F}_p(x,y) - \tilde{F}_q(x,y) &= \frac{1}{n}\sum_{i=1}^n\alpha_i(x)  \Big ( \tilde{d}_p(y,\tilde{y}_i) - \tilde{d}_q(y,\tilde{z}_i) \Big ) 
     = \frac{1}{n}\sum_{i=1}^n\alpha_i(x) \xi(y,\tilde{y}_i,\tilde{z}_i).
\end{align*}
Thus, 
\begin{align*}
    \E_{\tilde{Y},\tilde{Z}|Y=y_i} \big[ \tilde{d}_p(y,\tilde{Y}) - \tilde{d}_q(y,\tilde{Z})\big]  &= \E_{\tilde{Z}|Y=y_i}[\tilde{d}_q(y,\tilde{Y})\big]- \E_{\tilde{Z}|Y=y_i}[\tilde{d}_q(y,\tilde{Z})\big] \\
     &= d_\cY^2(y,y_i) - d^2_\cY(y,y_i) = 0
\end{align*}
Finally $\E_{\tilde{Y},\tilde{Z}}[\xi(y,\tilde{Y},\tilde{Z})] = 0$.

Next,
\begin{align*}
    \tilde{d}_p(y,\tilde{y}_i) - \tilde{d}_q(y,\tilde{z}_i) &\le |\tilde{d}_p(y,\tilde{y}_i) - \tilde{d}_q(y,\tilde{z}_i)| \\
    &\le |\tilde{d}_p(y,\tilde{y}_i) -\tilde{d}_p(y,\tilde{z}_i) + \tilde{d}_p(y,\tilde{z}_i) - \tilde{d}_q(y,\tilde{z}_i)|\\
    &\le |\tilde{d}_p(y,\tilde{y}_i) - d_{\cY}^2(y,\tilde{z}_i) + d_{\cY}^2(y,\tilde{z}_i) - \tilde{d}_p(y,\tilde{z}_i) + \tilde{d}_p(y,\tilde{z}_i) - \tilde{d}_q(y,\tilde{z}_i)| \\
    &\le |\tilde{d}_p(y,\tilde{y}_i) - d_{\cY}^2(y,\tilde{z}_i) |  + |d_{\cY}^2(y,\tilde{z}_i) - \tilde{d}_p(y,\tilde{z}_i)| + |\tilde{d}_p(y,\tilde{z}_i) - \tilde{d}_q(y,\tilde{z}_i)| \\
    &\le 2 c_1 + |\tilde{d}_p(y,\tilde{z}_i) - \tilde{d}_q(y,\tilde{z}_i)| \\
    &\le 2 c_1 + c_2.
\end{align*}
The first two terms are upper bounded as in Lemma \ref{lem:Fx-Concentration} and the last term is bounded using Lemma~\ref{lemma:bias-p-q}.
Since $\alpha_i(x) \le 1$ and $|\xi(y,\tilde{y}_i,\tilde{z}_i)|$ are upper bounded by $2c_1 + c_2$ as shown above, we have that |$\alpha_i(x)\cdot\xi(y,\tilde{y}_i,\tilde{z}_i)| \le 2c_1 + c_2$. 
\end{proof}

\begin{restatable}[]{lemma}{lemFuncDistBound}
Let $\hat{f}_p$ be the minimizer as defined in \eqref{eq:f-hat-q} over the noisy labels drawn from $\bP$, and let $\hat{f}_q$ (defined in eq. \eqref{eq:f-hat-q}) be the minimizer over the noisy labels obtained from conditional distribution $\bQ$. Then with high probability, 
\begin{equation}
    d^2_{\cY}\big( \hat{f}_q(x), \hat{f}(x) \big) \le \tilde{\cO} \Big( \frac{1}{\beta}\big(3c_1+c_2\big)\sqrt{\frac{1}{n} }   \Big) \qquad \forall x \in \cX .
\end{equation}
\label{thm:dist_minimizers_q_clean}
\end{restatable}
\begin{proof}
Let $ t_1 =\tilde{\cO} \Big( c_1 \sqrt{\frac{1}{n} }\Big) \Big) $ and $t_2 =   \tilde{\cO}\Big( (2c_1 +c_2)\sqrt{\frac{1}{n}}\Big) \Big)$, then combining Lemma \ref{lem:Fx-Concentration-PQ} and \ref{lem:Fx-Concentration} we have,
\begin{align*} F \big(x,f(x) \big) -t_1 - t_2 \le \tilde{F}_q\big(x,f(x)\big) \le F\big(x,f(x)\big) + t_1 + t_2 . \end{align*} 
Then following same argument as in Lemma \ref{lem:f_p_and_f_dist}, we get the result.
\end{proof}

The following lemmas bound the estimation error between noise matrices $\bP$ and $\bQ$ using the estimation error in the canonical parameters.
\begin{restatable}{lemma}{}
The posterior distribution function $P_{\bth}(Y=y|\Lambda=\Lambda^{u})$ is  $(2,\ell_{\infty})-$Lipshcitz continuous in $\bth$ for any $ y\in \cY \text{ and } \Lambda^{u} \in \cY^m.$
  $$ | P_{\bth_1}(Y=y|\Lambda=\Lambda^{u}) -  P_{\bth_2}(Y=y|\Lambda=\Lambda^{u})| \le 2 ||\bth_1 - \bth_2||_{\infty}  \qquad \forall \bth_1,\bth_2 \in \R^m .  $$
  \label{lem:lipschitzPosterior}
\end{restatable}
\begin{proof}
   Recall the definition of the posterior distribution,
   \begin{align*} P_{\bth}(Y=y_i|\Lambda=\Lambda^{u}) = \frac{p(Y=y_i) P_{\bth}(\Lambda=\Lambda^u | Y=y_i)}{\sum_{y_j \in \cY} p(Y=y_j) P_{\bth}(\Lambda=\Lambda^u | Y=y_j)} .
   \end{align*}
   
   For convenience let $\bd^{(u,i)} \in \R^m$ be such that its $a^{th}$ entry $\bd^{(u,i)}_a = d_{\cY}^2(\Lambda^u_{a},y_i)$
   \begin{align*} P_{\bth}(Y=y_i|\Lambda=\Lambda^{u}) = \frac{P(Y=y_i) \exp(-\bth^T \bd^{(u,i)} )}{ \sum_{y_j \in \cY} P(Y=y_j) \exp(-\bth^T \bd^{(u,j)})} .
   \end{align*}
   
   Let $Z_2(\bth) =\sum_{y_j \in \cY} P(Y=y_j) \exp(-\bth^T \bd^{(u,j)}) ,$ then 
   
   $$ -\nabla_{\bth} \log(Z_2(\bth)) = \frac{\sum_{y_j \in \cY}\bd^{(u,j)} P(Y=y_j) \exp(-\bth^T \bd^{(u,j)})}{Z_2(\bth)} = \E_{Y|\Lambda}[\bd].$$
   Since distances are upper bounded by 1, $||\bd||_{\infty} \le  1$, so $||\E_{Y|\Lambda}[\bd]||_{\infty} \le 1.$\\
   Now,
   \begin{align*} \nabla_{\bth} \log \big( P_{\bth}(Y=y|\Lambda=\Lambda^{u}) \big ) =   - \bd^{(u,i)} - \nabla_{\bth}\log(Z_2(\bth)) .
   \end{align*}
   Thus $|| \nabla_{\bth} \log \big( P_{\bth}(Y=y|\Lambda=\Lambda^{u}) \big ) ||_{\infty} \le 2$. \\
   $$\implies |\log \big( P_{\bth_1}(Y=y|\Lambda=\Lambda^{u}) \big ) - \log \big( P_{\bth_2}(Y=y|\Lambda=\Lambda^{u}) \big )| \le 2 ||\bth_1 - \bth_2||_{\infty} . $$
   Using the fact that for any $t_1,t_2 \in [0,1]$ $|t_1 - t_2| \le | \log(t_1) - \log(t_2)|, $ gives us the result. 
   
\end{proof}

\begin{restatable}{lemma}{}
The distribution function $P_{\bth}(\Lambda=\Lambda^{u} |Y=y)$ is  $(2,\ell_{\infty})-$Lipshcitz continuous in $\bth$ for any $ y\in \cY \text{ and } \Lambda^{u} \in \cY^m.$
  $$ | P_{\bth_1}(\Lambda=\Lambda^{u}|Y=y) -  P_{\bth_2}(\Lambda=\Lambda^{u}|Y=y)| \le 2 ||\bth_1 - \bth_2||_{\infty}  \qquad \forall \bth_1,\bth_2 \in \R^m  .$$ 
  \label{lem:lipschitzLike}
\end{restatable}

\begin{proof}
   Doing the same steps as in the  proof of Lemma \ref{lem:lipschitzPosterior} gives the result.
\end{proof}

\begin{restatable}{lemma}{} 
  For the noise distributions $\bP,\bQ$ in \eqref{eq:dist-P} with parameters $\bth$, $\hat{\bth}$ respectively and $\cY$ restricted only to the elements with non-zero prior probability, $\cY' = \{y \in \cY: P(Y=y)>0\}$ the following holds, 
  $$\max_{ij} |\bP_{ij} - \bQ_{ij}| \le 4\cdot k^m ||\bth -\hat{\bth}||_\infty \, .$$
  \label{lem:P-Q-bound}
\end{restatable}
\begin{proof}
    It is easy to see that for any two bounded functions $f_1,f_2$ with $|f_1(x)|\le 1, |f_2(x)|\le 1 $ and Lipschitz continuous with constants $L_1,L_2$, the product of them is  also Lipschitz continuous but with constant $L_1 + L_2$. Using this fact along with lemma \ref{lem:lipschitzPosterior} and lemma \ref{lem:lipschitzLike} gives the result,
    $$|\bP_{ij} - \bQ_{ij}| \le \sum_{\Lambda^{u} \in \cY'} |P_{\bth}(y_i|\Lambda^u)P_{\bth}(\Lambda^u|y_j) - P_{\hat{\bth}}(y_i|\Lambda^u)P_{\hat{\bth}}(\Lambda^u|y_j) | \le 4 \cdot k^m ||\bth-\hat{\bth}||_{\infty}.$$
\end{proof}
It is important to note that we are restricting the values of $y$ and $\lambda$ to $\cY'$ which is the set of $y$ with non-zero prior probability and by our assumption it is small. 

Finally, we restate and prove our generalization error result:
\thmGenErrTrueNoise*

\begin{proof}
Recall the definition of risk function,
   \[ R(f) = \E_{x,y}\big[ d^2_{\cY}\big(f(x),y \big)\big].\] 
   \begin{align*} R(\hat{f}_q)  &=\E_{x,y}\big[ d^2_{\cY}\big(\hat{f}_q(x),y \big ) \big], \\
   &\le \E_{x,y}\big[ d^2_{\cY}\big(\hat{f}_q(x),\hat{f}(x) \big )  + d^2_{\cY}(\hat{f}(x),y) + 2 d_{\cY}(\hat{f}_q(x),\hat{f}(x))\cdot  d_{\cY}(\hat{f}(x),y)   \big], \\
   &= \E_{x}[d^2_{\cY}\big(\hat{f}_q(x),\hat{f}(x) \big )]  + R(\hat{f}) + \tilde{\cO}(n^{-1/4}),\\
   &\le  \tilde{\cO} \Big( \frac{1}{\beta}\big(c_1+c2\big)\sqrt{\frac{1}{n}}  +  \frac{c_2}{\beta} \epsilon  \Big) + R(\hat{f}) + \tilde{\cO}(n^{-1/4}).
   \end{align*} 
Using the result from \cite{Ciliberto2016},
   $$R(\hat{f}) \le R(f^*) + \cO(n^{-1/4}).$$
   Combining the two we get
   $$ R(\hat{f}_q) \le R(f^*) + \tilde{\cO}(n^{-1/4}) +  \tilde{\cO} \Big( \frac{1}{\beta}\big(c_1+c2\big)\sqrt{\frac{1}{n}}  +  \frac{c_3}{\beta} \epsilon)  \Big) . $$
   We get the end result by plugging in the bound on $ \epsilon = \max_{ij}|| \bP - \bQ|| $ from Lemma \ref{lem:P-Q-bound} 
   and the bound on parameter recovery error $||\bth - \hat{\bth}||_{\infty}$ from Theorem \ref{thm:meanfinite}. 
   
\end{proof}

\section{Proofs for Continuous Label Spaces} 
Next we present the proofs for the results in the continuous (manifold-valued) label spaces. We restate the first result on invariance:
\invariantlemma*

\begin{proof}
We start with the hyperbolic law of cosines, which states that
\[\cosh d(\lambda_a, \lambda_b) = \cosh d(\lambda_a, y) \cosh d(\lambda_b, y) + \sinh  d(\lambda_a, y) \sinh  d(\lambda_b, y) \cos \alpha,  \]
where $\alpha$ is the angle between the sides of the triangle formed by $(y, \lambda_)$ and $(y, \lambda_b)$. We can rewrite this as follows. Let $v_a = \log_y(\lambda_a)$, $v_b = \log_y(\lambda_b)$ be tangent vectors in $T_yM$. Then, 
\[\cosh d(\lambda_a, \lambda_b) = \cosh d(\lambda_a, y) \cosh d(\lambda_b, y) + (\sinh  \|v_a\| \sinh  \|v_b\|) \langle \frac{v_a}{\|v_a\|}, \frac{v_b}{\|v_b\|} \rangle.  \]
Next, we take the expectation conditioned on $y$. The right-most term is then
\begin{align*}&\mathbb{E}[ (\sinh  \|v_a\| \sinh  \|v_b\|) \langle \frac{v_a}{\|v_a\|}, \frac{v_b}{\|v_b\|} \rangle | y ] \\
&\qquad= \mathbb{E}[ (\sinh  \|v_a\| \sinh  \|v_b\|) |y] \mathbb{E}[\langle \frac{v_a}{\|v_a\|}, \frac{v_b}{\|v_b\|} \rangle | y ] \\
&\qquad= 0,
\end{align*}
where the last equality follows from the fact that $v_a$ and $v_b$ are independent conditioned on $y$.
This leaves us with the $\cosh$ product terms. Taking expectation again with respect to $y$ gives the result.

The spherical version of the result is nearly identical, replacing hyperbolic sines and cosines with sines and cosines, respectively.
\end{proof}

Note, in addition, that it is easy to obtain a version of this result for curvatures that are not equal to $-1$ in the hyperbolic case (or $+1$ in the spherical case).

We will use this result for our consistency result, restated below.
\contPE*

\begin{proof}
First, we will condition on the event that the observed outputs have maximal distance (i.e., diameter) $\Delta$. This implies that our statements hold with high probability. Then, we use McDiarmid's inequality. For each pair of distinct LFs $a,b$, we have that
\[P\left(\frac{1}{n}|\sum_{i=1}^n \cosh(d(\lf_{a,i}, \lf_{b,i})) - \mathbb{E}\cosh(d(\lf_a, \lf_b))| \geq t \right) \leq 2 \exp \left(-\frac{2nt^2}{\cosh(\Delta)}   \right),\]

Integrating the expression above in $t$, we obtain
\begin{align}
\mathbb{E}|\hat{\mathbb{E}}\cosh(d(\lf_a, \lf_b)) - \mathbb{E}\cosh(d(\lf_a, \lf_b))| \leq \frac{\sqrt{\pi \cosh(\Delta)}}{\sqrt{2n}}.
\label{eq:mcd}
\end{align}

Next, we use this to control the gap on our estimator. Recall that using the triplet approach, we estimate
\[\hat{\mathbb{E}} \cosh(d(\lf_a, y)) = \sqrt{\frac{\hat{\mathbb{E}}\cosh d(\lf_a,\lf_b) \hat{\mathbb{E}} \cosh d(\lf_a, \lf_c)}{(\hat{\mathbb{E}}d(\lf_b, \lf_c))^2}}. \]

For notational convenience, we write $\nu(a)$ for $\mathbb{E}(\cosh(d(\lf_a, y)))$, $\hat{\nu}(a)$ for its empirical counterpart, and $\nu(a,b)$ and $\hat{\nu}(a,b)$ for the versions between pairs of LFs $a,b$. Then, the above becomes
\[\hat{\nu}(a) = \sqrt{\frac{\hat{\nu}(a,b)\hat{\nu}(a,c)}{(\hat{\nu}(b,c))^2}}.\]

Note that $\cosh(x) \geq 1$, so that $\hn(a,b) \geq 1$ and similarly for the empirical versions. We also have that $\hn(a,b) \leq \cosh(\Delta)$. With this, we can begin our perturbation analysis. Applying Lemma~\ref{lem:invariant}, we have that
\begin{align*}
\mathbb{E}|\hn(a) - \nu(a)| &= \mathbb{E}\left|\sqrt{\frac{\hn(a,b)\hn(a,c)}{\hn(b,c)^2}} - \sqrt{\frac{\nu(a,b)\nu(a,c)}{\nu(b,c)^2}} \right| \\
&=  \mathbb{E}\left|\sqrt{\frac{\hn(a,b)\hn(a,c)}{\hn(b,c)^2}} - \sqrt{\frac{\nu(a,b)\hn(a,c)}{\hn(b,c)^2}} + \sqrt{\frac{\nu(a,b)\hn(a,c)}{\hn(b,c)^2}} - \sqrt{\frac{\nu(a,b)\nu(a,c)}{\nu(b,c)^2}} \right| \\
&\leq \mathbb{E}\left|\sqrt{\frac{\hn(a,b)\hn(a,c)}{\hn(b,c)^2}} - \sqrt{\frac{\nu(a,b)\hn(a,c)}{\hn(b,c)^2}} \right| + \mathbb{E}\left| \sqrt{\frac{\nu(a,b)\hn(a,c)}{\hn(b,c)^2}} - \sqrt{\frac{\nu(a,b)\nu(a,c)}{\nu(b,c)^2}} \right| \\
&= \mathbb{E} \left| \sqrt{\frac{\hn(a,c)}{\hn(b,c)^2}}(\sqrt{\hn(a,b)} - \sqrt{\nu(a,b)})\right| + 
\mathbb{E}\left| \sqrt{\frac{\nu(a,b)\hn(a,c)}{\hn(b,c)^2}} - \sqrt{\frac{\nu(a,b)\nu(a,c)}{\nu(b,c)^2}} \right| \\
 &\leq \frac{\sqrt{\pi} \cosh(\Delta^2)}{\sqrt{2n}} + 
\mathbb{E}\left| \sqrt{\frac{\nu(a,b)\hn(a,c)}{\hn(b,c)^2}} - \sqrt{\frac{\nu(a,b)\nu(a,c)}{\nu(b,c)^2}} \right|.
\end{align*}
To see why the last step holds, note that $\sqrt{\hn(a,c)} \leq \sqrt{\cosh(\Delta)}$, while $\hn(b,c) \geq 1$. Next, for $\alpha, \beta \geq 1$, $\sqrt{\alpha}-\sqrt{\beta} = \frac{\alpha - \beta}{\sqrt{\alpha}-\sqrt{\beta}} \leq \alpha - \beta$. This means that 
$\mathbb{E}|\sqrt{\hn(a,b)} - \sqrt{\nu(a,b)}| \leq \mathbb{E}|\hn(a,b) - \nu(a,b)| \leq \frac{\sqrt{\pi \cosh(\Delta^)}}{\sqrt{2n}}$ using \eqref{eq:mcd}.

Now we can continue, adding and subtracting as before. We have that
\begin{align*}
\mathbb{E}&\left| \sqrt{\frac{\nu(a,b)\hn(a,c)}{\hn(b,c)^2}} - \sqrt{\frac{\nu(a,b)\nu(a,c)}{\nu(b,c)^2}} \right| \\
&\qquad \leq \mathbb{E}\left| \sqrt{\frac{\nu(a,b)\hn(a,c)}{\hn(b,c)^2}} - \sqrt{\frac{\nu(a,b)\nu(a,c)}{\hn(b,c)^2}} \right| + \mathbb{E}\left| \sqrt{\frac{\nu(a,b)\nu(a,c)}{\hn(b,c)^2}} - \sqrt{\frac{\nu(a,b)\nu(a,c)}{\nu(b,c)^2}} \right| \\
&\qquad \leq \frac{\sqrt{\pi} \cosh(\Delta)}{\sqrt{2n}} + \mathbb{E}\left| \sqrt{\frac{\nu(a,b)\nu(a,c)}{\hn(b,c)^2}} - \sqrt{\frac{\nu(a,b)\nu(a,c)}{\nu(b,c)^2}} \right| \\
&\qquad \leq \frac{\sqrt{\pi} \cosh(\Delta)}{\sqrt{2n}} + \frac{\sqrt{\pi}\cosh(\Delta)^{3/2}}{\sqrt{2n}}.
\end{align*}
Putting it all together, with probability at least $1-\delta$,
\begin{align}
\label{eq:errcosh}
    \mathbb{E}|\hat{\mathbb{E}} \cosh(d(\lf_a, y)) - \mathbb{E}\cosh(d(\lf_a, y))| \leq \frac{2 \sqrt{\pi} \cosh(\Delta) + \sqrt{\pi}\cosh(\Delta)^{3/2}}{\sqrt{2n}}.
\end{align}

Next, recall that $C_0$ satisfies $\mathbb{E}|\hat{\mathbb{E}}\cosh(d(\lf_a, \lf_b)) - \mathbb{E}\cosh(d(\lf_a, \lf_b))| \geq C_0 \mathbb{E}|\hat{\mathbb{E}}d(\lf_a, \lf_b)) - \mathbb{E}d(\lf_a, \lf_b)|$. Thus, 
\[
\mathbb{E}|\hat{\mathbb{E}}d^2(\lf_a, y) - \mathbb{E}d^2(\lf_a, y)| \leq \frac{2 \sqrt{\pi} \cosh(\Delta) + \sqrt{\pi}\cosh(\Delta)^{3/2}}{C_0\sqrt{2n}}.
\]

This concludes the proof.
\end{proof}


Next, we will prove a simple result that is needed in the proof of Theorem~\ref{thm:contGen}. Consider the distribution $P$ of the quantities $\alpha(x)(y) d_{\cY}^2(z, y)$ for some fixed $z \in \M$. We can think of this as the population-level version of sample distances that are observed in the supervised version of the problem. We do not have access to it in our approach; it will be used only as an object in our proof. Recall we set $q = \argmin_{z \in \cY} \mathbb{E} [\alpha(x)(y) d_{\cY}^2(z, y)]$ to be the population-level minimizer. Here we use the notation $\alpha(x)(y)$ to denote the corresponding kernel value at a point $y$. Finally, let us denote $P'$ to be the distribution over the quantities $\alpha(x)(y) \sum_{a=1}^m \beta^2_a d^2_{\cY}(z, \lf_{a,i})$. 

\begin{lemma}
Let the distributions $P$ and $P'$ be defined as above, with $q$ the minimizer of $\mathbb{E}_P [\alpha(x)(y) d_{\cY}^2(z, y)]$. Suppose that Assumptions~\ref{ass:bhf} and ~\ref{ass:sym} hold. Then, $q$ is also the minimizer of $\mathbb{E}_P'[\alpha(x)(y) \sum_{a=1}^m \beta^2_a d^2_{\cY}(z, \lf_{a,i})]$.
\label{lem:samecenter}
\end{lemma}

\begin{proof}
We will use a simple symmetry argument. First, note that we can write $q$ in the following way,
\[q = \argmin_{z \in \cY} \int_{T_q \M} \alpha(x)(\log_q(v)) d_{\cY}^2(z, \exp_q(v)) dP.\]

Since $\M$ is a symmetric manifold, if $v \in T_q \M$, there is an isometry sending $v$ to $-v \in T_q \M$. Using this isometry and Assumption~\ref{ass:sym}, we can also write 
\[q = \argmin_{z \in \cY} \int_{T_q \M} \alpha(x)(\log_q(-v)) d_{\cY}^2(z, \exp_q(-v)) dP.\]

Our approach will be to formulate similar symmetric expressions for the minimizer, but this time for the loss over the distribution $P'$. We will then be able to show, using triangle inequality, that $q$ remains the minimizer.

We can similarly express the minimizer of the loss for $P'$ as
\[\argmin_{z \in \cY} \int_{T_q \M} \int_{(T_{\exp_q(v)} \M)^{\otimes m})} \alpha(x)(\log_q(v)) \sum_{a=1}^m \beta^2_a d_{\cY}^2 (z, \exp_{\exp_q(v)}(v_a)) dP'.\]
Here we have broken down the expectation over $P'$ by applying the tower law; the inner expectation is conditioned on point $\exp_q(v)$ and runs over the labeling function outputs $\lambda_1, \ldots, \lambda_m$.

Again using Assumption~\ref{ass:sym}, we can write the minimizer for the loss over $P'$ as $\argmin_{z \in \cY} F'(z)$, where
\[F'(z) = \int_{T_q \M} \int_{(T_{\exp_q(-v)} \M)^{\otimes m})} \alpha(x)(\log_q(-v)) \sum_{a=1}^m \beta^2_a d_{\cY}^2 (z, \exp_{\exp_q(-v)}(-v_a)) dP'.\]

Thus we can also write the minimizer as $\argmin_{z \in \cY} F'(z)$, where
\[F'(z) =  \int_{T_q \M} \int_{(T_{\exp_q(-v)} \M)^{\otimes m})} \alpha(x)(\log_q(-v)) \sum_{a=1}^m \beta^2_a d_{\cY}^2 (z, \exp_{\exp_q(-v)}(-v_a)) dP'.\]

With this, we can write
 \begin{align*}
F'(z) &= \frac{1}{2} \left( \int_{T_q \M} \int_{(T_{\exp_q(v)} \M)^{\otimes m})} \alpha(x)(\log_q(v)) \sum_{a=1}^m \beta^2_a d_{\cY}^2 (z, \exp_{\exp_q(v)}(v_a)) dP' \right. \\
&\left.  \qquad + \int_{T_q \M} \int_{(T_{\exp_q(-v)} \M)^{\otimes m})} \alpha(x)(\log_q(-v)) \sum_{a=1}^m \beta^2_a d_{\cY}^2 (z, \exp_{\exp_q(-v)}(-v_a)) dP' \right) \\
&= \frac{1}{2} \left( \int_{T_q \M} \int_{(T_{\exp_q(v)} \M)^{\otimes m})} \alpha(x)(\log_q(v)) \sum_{a=1}^m \beta^2_a d_{\cY}^2 \left((z, \exp_{\exp_q(v)}(v_a)) \right. \right. \\
&\qquad +\left. \left. d_{\cY}^2 (z, \exp_{\exp_q(-v)}(PT_{\exp_q(v) \rightarrow \exp_q(-v)}(-v_a)) ) \right) dP' \right),
 \end{align*}
where $PT_{p \rightarrow s}$ denotes parallel transport from $p$ to $s$.

Note that $q$ is on the geodesic between $\exp_{\exp_q(v)}(v_a)$ and $\exp_{\exp_q(-v)}(PT_{\exp_q(v) \rightarrow \exp_q(-v)}(-v_a))$. We exploit this fact by applying the following squared-distance inequality. For three points $p,s,z$, from the triangle inequality, 
\[d_{\cY}(p,z) + d_{\cY}(s,z) \geq d_{\cY}(p,s).\] 
Squaring both sides and applying \[d_{\cY}^2(p,z) + d_{\cY}^2(s,z) \geq 2d_{\cY}(p,z)d_{\cY}(s,z),\] 
we obtain that 
\[2(d_{\cY}^2(p,z) + d_{\cY}^2(s,z)) \geq d_{\cY}^2(p,s),\]
so that 
\[d_{\cY}^2(p,z) + d_{\cY}^2(q,z) \geq \frac{1}{2}d_{\cY}^2(p,q).\]
Setting $p$ to be $\exp_{\exp_q(v)}(v_a)$ and $s$ to be $\exp_{\exp_q(-v)}(PT_{\exp_q(v) \rightarrow \exp_q(-v)}(-v_a))$ in the above gives
\begin{align*}
F'(z)& \geq \frac{1}{2} \left( \int_{T_q \M} \int_{(T_{\exp_q(v)} \M)^{\otimes m})} \alpha(x)(\log_q(v)) \sum_{a=1}^m \beta^2_a \right. \\ 
&\qquad \left. \frac{1}{2} d_{\cY}^2(\exp_{\exp_q(v)}(v_a),\exp_{\exp_q(-v)}(PT_{\exp_q(v) \rightarrow \exp_q(-v)}(-v_a))) dP' \right) .
\end{align*}
Now we can apply the fact that $q$ is on the geodesic to rewrite this as
 \begin{align*}
F'(z)& \geq \frac{1}{2} \left( \int_{T_q \M} \int_{(T_{\exp_q(v)} \M)^{\otimes m})} \alpha(x)(\log_q(v)) \sum_{a=1}^m \beta^2_a \frac{1}{2} 4d_{\cY}^2(q, \exp_{\exp_q(v)}(v_a)) dP' \right) .
\end{align*}
This is because the length of the geodesic connecting $\exp_{\exp_q(v)}(v_a)$ and $\exp_{\exp_q(-v)}(PT_{\exp_q(v) \rightarrow \exp_q(-v)}(-v_a))$ is twice that of the geodesic connecting $\exp_{\exp_q(v)}(v_a)$ to $q$.

Thus, we have
 \begin{align*}
F'(z)& \geq F'(q), 
\end{align*}
and we are done.
\end{proof}

Finally, this enables us to prove our main result, Theorem~\ref{thm:contGen}, restated below:
\contGen

\begin{proof}
We use Lemma~\ref{lem:samecenter} and compute a bound on the expected distance from the empirical estimates to the common center. In both cases, the approach is nearly identical to that of \cite{Stromme} (proof of Theorem 3.2.1); we include these steps for clarity. Suppose that the minimum and maximum values of $\alpha$ are $\alpha_{\min}$ and  $\alpha_{\max}$, respectively. 

Then, letting we have that, using the hugging function assumption
\[ \|\log_{q}(\hat{f}(x)) - \log_q(y_i)\|^2  \leq k_{\min} d^2_{\cY}(q, \hat{f}(x)) + d_{\cY}^2(\hat{f}(x), y_i).\]

We also have that
\[ \|\log_{q}(\hat{f}(x)) - \log_q(y_i)\|^2 = d_{\cY}^2(q, \hat{f}(x)) - 2\langle \log_q(\hat{f}(x)), \log_q(y_i) \rangle + d_{\cY}^2(q, y_i). \]

Then,
\[(1-k_{\min})d^2_{\cY}(q, \hat{f}(x)) \leq 2\langle \log_q(\hat{f}(x)), \log_q(y_i) \rangle + d_{\cY}^2(\hat{f}(x), y_i) - d_{\cY}^2(q, y_i).\]

Now, multiply each of the equations by $\alpha_i$ and sum over them. In that case, the different on the right side is non-positive, as $\hat{f}(x)$ is the empirical minimizer. This yields
\[\sum_{i=1}^n \alpha(x)_i (1-k_{\min})d^2_{\cY}(q, \hat{f}(x)) \leq \sum_{i=1}^n \alpha(x)_i 2\langle \log_q(\hat{f}(x)), \log_q(y_i) \rangle .\]

Using the minimum and maximum values of $\alpha$, and setting $\bar{q} = \sum_{i=1}^n \log_q(y_i)$, we get
\[\alpha_{\min} (1-k_{\min})d^2_{\cY}(q, \hat{f}(x)) \leq 2\alpha_{\max} \langle \log_q(\hat{f}(x)), \bar{q} \rangle.\]

We can apply Cauchy-Schwarz, simplify, then square, obtaining
\[\alpha_{\min}^2 (1-k_{\min})^2d^2_{\cY}(q, \hat{f}(x)) \leq 4\alpha_{\max}^2 \|\bar{q}\|^2.\]

What remains is to take expectation and use the fact that the tangent vectors summed up to form $\bar{q}$ are independent. This yields
\[\alpha_{\min}^2 (1-k_{\min})^2 \mathbb{E} d^2_{\cY}(q, \hat{f}(x)) \leq 4\alpha_{\max}^2 \frac{\sigma_o^2}{n}.\]

Thus we obtain
\[\alpha_{\min}^2 (1-k_{\min})^2 \mathbb{E} d^2_{\cY}(q, \hat{f}(x)) \leq 4\alpha_{\max}^2 \frac{\sigma_o^2}{n},\]
or
\begin{align}
  \mathbb{E} d^2_{\cY}(q, \hat{f}(x)) \leq 4\frac{\alpha_{\max}^2}{\alpha_{\min}^2} \frac{\sigma_o^2}{n k_{\min}}.  
\label{eq:disthat}
\end{align}

We use the same approach, but this apply it to the $m \times n$ points given by the LFs drawn from distribution $P'$. This yields 
\[\alpha_{\min}^2 (1-k_{\min})^2 \mathbb{E} d^2_{\cY}(q, \tilde{f}(x)) \leq 4\alpha_{\max}^2 \frac{\sum_{i=1}^m \beta_a^2 \sigma_{a}^2}{mn},\]
where $\sigma_{a}^2$ corresponds to the expected squared distance for LF $a$ to $q$. We bound this with triangle inequality, obtaining $\sigma_{a}^2 \leq 2 \sigma_o^2 + 2 \hat{\mu}_a^2$, so that
\[\alpha_{\min}^2 (1-k_{\min})^2 \mathbb{E} d^2_{\cY}(q, \tilde{f}(x)) \leq 8\alpha_{\max}^2 \frac{\sum_{i=1}^m \beta_a^2 (\sigma_o + \hat{\mu}_a^2}{mn},\]
or, 
\begin{align}
\mathbb{E} d^2_{\cY}(q, \tilde{f}(x)) \leq 8\frac{\alpha_{\max}^2}{\alpha_{\min}^2} \frac{\sum_{i=1}^m \beta_a^2 (\sigma_o + \hat{\mu}_a^2}{mn k_{\min}}.
\label{eq:disttilde}
\end{align}

Now, again using triangle inequality,
\[\mathbb{E} d^2_{\cY}(\hat{f}(x), \tilde{f}(x)) \leq 2\mathbb{E} d^2_{\cY}(q, \hat{f}(x)) + 2 \mathbb{E} d^2_{\cY}(q, \tilde{f}(x)).  \]

Plugging \eqref{eq:disttilde} and \eqref{eq:disthat} into this bound produces the result.
\end{proof}

\section{Additional Details on Continuous Label Space }
We provide some additional details on the continuous (manifold-valued) case.

\paragraph{Computing $\Delta(\delta)$} In Theorem~\ref{thm:contPE}, we stated the result in terms of $\Delta(\delta)$, a quantity that trades off the probability of failure $\delta$ for the diameter of the largest ball that contains the observed points. Note that if we fix the curvature of the manifold, it is possible to compute an exact bound for this quantity by using formulas for the sizes of balls in $d$-dimensional manifolds of fixed curvature. 

\paragraph{Hugging number} Note that it is possible to derive a lower bound on the hugging number as a function of the curvature. The way to do so is to use \emph{comparison theorems} that upper bound triangle edge lengths with those of larger-curvature triangles. This makes it possible to establish a concrete value for $k_{\min}$ as a function of the curvature.

We note, as well, that an upper bound $k_{\max}$ on the hugging number can be obtained by a simple rearrangement of Lemma 6 from \cite{Zhang16}. This result follows from a curvature lower bound based on hyperbolic law of cosines; the bound we describe follows from the opposite---an upper bound based on spherical triangles.

\paragraph{$\beta$ Weights and Suboptimality} An intuitive way to think of the estimator we described is the following simple Euclidean version. Suppose we have labeling functions $\lambda_1, \ldots, \lambda_m$ that are equal to $y + \varepsilon_a$, where $\varepsilon_a \sim \mathcal{N}(0, \sigma^2_a)$. In this case, if we seek an unbiased estimator with lowest variance, we require a set of weights $\beta_a$ so that $\sum_a \beta_a = 1$ and $\text{Var}[\frac{1}{m} \sum_{a=1}^m \beta_a \lambda_a]$ is minimized. It is not hard to derive a closed-form solution for the $\beta_a$ coefficients as a function of the terms $\sigma^2_a$.

Now, suppose we use the same solution, but with noisy estimates $\hat{\sigma}^2$ instead. Our weights $\hat{\beta}$ will yield a suboptimal variance, but this will not affect the scaling of the rate in terms of the number of samples $n$. 

\section{Extended Background on Pseudo-Euclidean Embeddings}
%
 We provide some additional background on pseudo-metric spaces and pseudo-Euclidean embeddings. Our roadmap is as follows. First, we note that pseudo-Euclidean spaces are a particular kind of pseudo-metric space, so we provide additional background and formal definitions for these pseudo-metric spaces. Afterwards, we explain some of the ideas behind pseudo-Euclidean spaces, comparing them to standard Euclidean spaces in the context of embeddings.

\subsection{Pseudo-metric Spaces} 
Pseudo-metric spaces generalize metric spaces by removing the requirement that pairs of points at distance zero must be identical:
\begin{definition}(\textbf{Pseudo-metric Space})
   A set $\cY$ along with a distance function $d_{\cY} : \cY\times \cY \mapsto \R^+$ is called pseudo-metric space if $d_{\cY}$ satisfies the following conditions,
     \begin{align}
         \forall \by,\bz \in \cY &\quad  \quad d_\cY(\by,\bz) = d_{\cY}(\by,\bz) \\ \eqname{(Symmetry)} \\
         \forall \by\in \cY &\quad \quad d_\cY(\by,\by)=0 \\ \eqname{(Reflexivity)} \\
         \forall \bx, \by, \bz \in \cY &\quad \quad d_\cY(\by,\bx) \leq d_\cY(\by,\bz) + d_\cY(\bx,\bz) \\ \eqname{(Triangle Inequality)}
     \end{align}
These spaces have additional flexibility compared to standard metric spaces: note that while $d(y,y) = 0$, $d(x,y) = 0$ does not imply that $x$ and $y$ are identical. The downside of using such spaces, however, is that conventional algebra may not produce the usual results. For example, limits where the distance between a sequence of points and a particular point tends to zero do not convey the same information as in standard metric spaces. However, these odd properties do not concern us, as we only use the spaces for representing a set of distances from our given metric space.
   
\end{definition}
A finite pseudo-metric space has $|\cY|<\infty$. 

\subsection{Pseudo-Euclidean Spaces}
The following definitions are for \emph{finite-dimensional} vector spaces defined over the field $\R$.
\begin{definition} (\textbf{Symmetric Bilinear Form / Generalized Inner Product}) For a vector space $\cY$ over the field $\R$, a symmetric bilinear form  is a function $\phi: \cY\times \cY \mapsto \R$ satisfying the following properties $\forall y_1,y_2,z,y \in \cY, c \in \R$:
\begin{enumerate}
    \item[P1)] $\phi(y_1+y_2,y) = \phi(y_1,y) + \phi(y_2,y),$
    \item[P2)] $\phi(cy,z) = c\phi(y,z),$
    \item[P3)] $\phi(y,z) = \phi(z,y)$.
\end{enumerate}
   
\end{definition}

\begin{definition}(\textbf{Squared Distance w.r.t.} $\phi$) Let $V$ be a real vector space equipped with generalized inner product $\phi$, then the squared distance w.r.t. $\phi$  between any two vectors $\by,\bz \in V$ is defined as,
 $$||\by-\bz||_{\phi}^2 \ldef \phi(\by-\bz,\by-\bz)$$
\end{definition}
This definition also gives a notion of squared length for every $\by\in V$,
$$||\by||_\phi^2 \ldef \phi(\by,\by)$$
The inner product can also be expressed in terms of a basis of the vector space $V$. Let the dimension of $\cY$ be $d$, and $\{\bb_i\}_{i=1}^d$ be a basis of $\cY$, then for any two vectors $\by=[y_1,\ldots y_d], \bz = [z_1,\ldots z_d] \in V$,
  $$\phi(\by,\bz) = \sum_{i=1}^d\sum_{j=1}^d y_i z_i \phi(\bb_i,\bb_j)$$

The matrix $\bM(\phi) \ldef  [\phi(\bb_i,\bb_j)]_{1\le i,j\le d} $ is called \emph{the matrix of} $\phi$ w.r.t the basis $\{\bb_i\}_{i=1}^d$
It gives a convenient way to express the inner product as $\phi(\by,\bz) = \by^T\bM(\phi)\bz$. A symmetric bilinear form $\phi$ on a vector space of dimension $d$, is said to be \emph{non-degenerate} if the rank  of $\bM(\phi)$ w.r.t to some basis is equal to $d$.

Example: For the $d-$ dimensional euclidean space with standard basis and $\phi$ as dot product we get $\bM(\phi) = \bI_d$

\begin{definition}(\textbf{Pseudo-euclidean Spaces})
A real vector space $\psR$ of dimension $d=d^+ + d^-$, equipped with a non-degenerate symmetric bilinear form $\phi$ is called a pseudo-euclidean (or Minkowski) vector space of signature $(d^+,d^-)$ if the matrix of $\phi$ w.r.t a basis $\{\bb_i\}_{i=1}^d$ of $\psR$, is given as,
    $$\bM(\phi) = \begin{pmatrix}  \bI_{d^+} & \bf 0 \\ 
            \bf 0 & -\bI_{d^-}
    \end{pmatrix}_{d\times d} $$
 
\end{definition}
\paragraph{Embedding Algorithms}
The tool that ensures we can produce isometric embeddings is the following result:
\begin{restatable}[]{proposition}{pseEmbThm} (\cite{goldfarb85psuedo})
Let $\cY = \{y_0,\ldots y_k\}$ be a finite pseudo-metric space equipped with distance function $d_\cY$, and let $\bV = \{\bv_i,\ldots,\bv_k\}$ be a collection of vectors in $\R^{d^+,d^-}$. Then $\cY$ is isometrically embeddable in $\R^{d^+,d^-}$ if and only if,
\begin{equation}
    \langle \bv_i, \bv_j \rangle_{\phi} = \frac{1}{2} \Big (d^2_{\cY}(y_i,y_0) + d^2_{\cY}(y_j,y_0) - d^2_{\cY}(y_i,y_j) \Big ) \quad  \forall i,j \in [k]
\end{equation}
\end{restatable}

  This bilinear form is very similar to the one used for MDS embeddings \cite{KruskalWish1978}---it is closely related to the squared distance matrix. The main information needed is what the signature (i.e., how many positive, negative, and zero eigenvalues) of this bilinear form is. If the dimension of the pseudo-Euclidean space we choose to embed in is at least as large as the number of positive and negative eigenvalues, we can obtain isometric embeddings. Because we are working with finite metric spaces, this number is always finite, and, in fact, is never larger than the size of the metric space. This means we can always produce isometric embeddings.

The practical aspects of how to produce the embedding are shown in the first half of Algorithm~\ref{alg:finite}. The basic idea is to do an eigendecomposition and capture eigenvectors corresponding to the positive and negative eigenvalues. These allow us to perfectly reproduce the positive and negative components of the distances separately; the resulting distance is the difference between the two components. The process of performing the eigendecomposition is standard, so that the overall procedure has the same complexity as running MDS. Compare this to MDS: there, we only capture the eigenvectors corresponding to the positive eigenvalues and ignore the negative ones. Otherwise the procedure is identical.

We note that in fact it is possible to embed pseudo-metric spaces isometrically into pseudo-Euclidean spaces, but we never use this fact. Our only application of this tool is to embed conventional metric spaces. However, our results directly lift to this more general setting.

The idea of using pseudo-Euclidean spaces for embeddings that can then be used in kernel-based or other classifiers or other approaches to machine learning is not new. For example, \cite{Pekalska01} used these spaces for kernel-based learning, \cite{laub2006information} used them for generic pairwise learning, and \cite{Pekalska06} showed that they are among several non-standard spaces that provide high-quality representations. Our contribution is using these in the context of weak supervision and learning latent variable models.

\paragraph{Dimensionality} We also give more detail on the example we provided showing that pseudo-Euclidean embeddings can have arbitrarily better dimensionality compared to one-hot encodings. The idea here is simple. We start with a particular kind of tree with a root and three branches that are simply long chains (paths) and have $t$ nodes each, for a total of $3t+1$ nodes. One-hot encodings have dimension that scales with the number of nodes, i.e., dimension $3t+1$.

Pseudo-euclidean embeddings enable us to embed such a tree into a space of finite (and in fact, very small) dimension while preserving the shortest-hops distances between each pair of nodes in the graph. As described above, the key question is what the number of positive and negative eigenvalues for the squared distance matrix (and thus the bilinear form) is. Fortunately, for such graphs, the signature of the squared-distance matrix is known (Theorem 20 in \cite{Bapat16}). Applying this result shows that the pseudo-Euclidean dimension is just 3, a tiny fixed value regardless of the value of $t$ above.

\end{document}